\documentclass{scrartcl}
\usepackage[utf8]{inputenc}
\usepackage[small]{caption}
\usepackage{graphicx}
\usepackage{amsmath}
\usepackage{amsthm}
\usepackage{booktabs}
\usepackage{algorithm}
\usepackage{placeins}
\usepackage[switch]{lineno}
\usepackage[table]{xcolor}
\usepackage{pdflscape}

\usepackage{algpseudocode}

\usepackage{multirow}

\usepackage{xcolor}
\definecolor{darkgreen}{rgb}{0,0.5,0}
\usepackage{tikz}
\usepackage{amssymb}
\usepackage{makecell}
\usepackage{float}
\usepackage{geometry}

\usepackage{hyperref}
\hypersetup{
    colorlinks,
    linkcolor={red!50!black},
    citecolor={blue!50!black},
    urlcolor={blue!50!black}
}
\usepackage[nameinlink]{cleveref}

\AddToHook{shipout/lastpage}{\thispagestyle{empty}}

\newtheorem{proposition}{Proposition}

\newtheorem{definition}{Definition}

\newcommand{\RR}{\mathbb{R}}

\newcommand{\nonnegReals}{\RR_{\geq 0}}

\DeclareMathOperator{\Plane}{Plane}
\newcommand{\numPlaneInv}{\Plane_{\#}}
\newcommand{\numPlaneFun}[1]{\numPlaneInv(#1)}
\DeclareMathOperator{\Sphere}{Sphere}
\newcommand{\numSphereInv}{\Sphere_{\#}}
\newcommand{\numSphereFun}[1]{\numSphereInv(#1)}
\newcommand{\mbezout}{\textnormal{m-Bézout}}

\newcommand{\rigInv}{\mathcal{R}}

\newcommand{\nGraphs}{\mathcal{M}_n}

\DeclareMathOperator{\NAC}{NAC}
\newcommand{\numNACinv}{\NAC_{\#}}
\newcommand{\numNACfun}[1]{\numNACinv(#1)}

\newcommand{\LDP}[1]{\mathrm{LDP}(#1)}

\newcommand{\doi}[1]{DOI:\href{https://doi.org/#1}{#1}}

\colorlet{vcol}{black!95!white}
\colorlet{ecol}{black!70!white}
\definecolor{col1}{HTML}{BBBBBB}
\definecolor{col2}{HTML}{EE6677}
\definecolor{col3}{HTML}{228833}
\definecolor{col4}{HTML}{CCBB44}
\definecolor{col5}{HTML}{66CCEE}
\definecolor{col6}{HTML}{AA3377}
\definecolor{col7}{HTML}{4477AA}

\tikzstyle{vertex}=[fill=vcol,circle,inner sep=0pt, minimum size=4pt]
\tikzstyle{smallvertex}=[draw=black, circle, fill=white, line width=1pt, inner sep=0pt, minimum size=3pt]
\tikzstyle{lvertex}=[fill=white, draw=vcol,circle,inner sep=0.5pt, minimum size=7pt]
\tikzstyle{edge}=[line width=1.5pt,ecol]
\tikzstyle{thinedge}=[edge, line width=1.2pt]
\tikzstyle{graphIndication}=[line width=1pt, gray, densely dotted]

\newcommand{\robot}[4]{%
    \begin{scope}[shift={#1}, rotate=#2, scale=#3]
        \draw[thick, fill=gray!20] (0,0) circle (0.4);
        
        \draw[thick, fill=black!40] (-0.45,-0.3) rectangle (-0.25,0.3);
        \draw[thick, fill=black!40] (0.25,-0.3) rectangle (0.45,0.3);
        
        \node[lvertex] at (0,0) {\scriptsize $#4$};
    \end{scope}
}

\title{Learning Minimally Rigid Graphs \\ with High Realization Counts}

\author{%
    Oleksandr Slyvka%
    \thanks{Faculty of Information Technology, Czech Technical University in Prague, Czechia, \\\texttt{\{slyvkole, rodrigo.alves, jan.legersky\}@fit.cvut.cz}}
    \\[2pt]
    Rodrigo Alves$^*$
    \and
    Jan Rubeš%
    \thanks{ETH Z\"urich, Switzerland, \texttt{jrubes@ethz.ch}}
    \\[2pt]
    Jan Legerský$^*$}
\date{\small This is an extended version of the paper accepted to IJCAI 2026.}

\begin{document}
\maketitle

\begin{abstract}
For minimally rigid graphs, the same edge-length data can admit multiple realizations (up to translations and rotations). Finding graphs with exceptionally many realizations is an extremal problem in rigidity theory, but exhaustive search quickly becomes infeasible due to the super-exponential growth of the number of candidate graphs and the high cost of realization-count evaluation. We propose a reinforcement-learning approach that constructs minimally rigid graphs via $0$- and $1$-extensions, also known as Henneberg moves. We optimize realization-count invariants using the Deep Cross-Entropy Method with a policy parameterized by a Graph Isomorphism Network encoder and a permutation-equivariant extension-level action head. Empirically, our method matches the known optima for planar realization counts and \emph{improves the best known bounds} for spherical realization counts, yielding new record graphs.
\end{abstract}

\section{Introduction}

Consider a fixed team of $n$ robots moving collaboratively in the plane. Each robot is equipped with only a few sensors (typically far fewer than $n-1$) that measure and enforce fixed distances to \emph{some} of the others, thereby imposing a set of pairwise distance constraints on the formation. From these distances alone, can we guarantee that the formation cannot \emph{continuously} deform into a different shape while keeping all measured distances unchanged, so that the robots can reliably synchronize their motion? The sensing pattern is captured by a \emph{measurement graph} whose vertices represent robots and whose edges indicate pairs for which a distance is available. Mathematically, this is formalized by notions from Rigidity Theory~\cite{handbook}: a planar realization of the measurement graph is called \emph{rigid} if it admits no nontrivial continuous deformation that preserves all edge lengths; otherwise, it is called \emph{flexible}; see \Cref{fig:rigid_vs_flexible}.

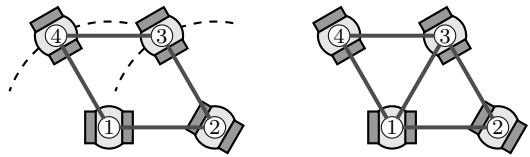
\begin{figure}[h]
    \centering
    \begin{tikzpicture}[scale=0.7]
        \node[lvertex] (A) at (0,0) {};
        \node[lvertex] (B) at (2,0) {};
        
        \draw[thick,dashed] ([shift=(80:2cm)] A) arc (80:160:2cm);
        \draw[thick,dashed] ([shift=(80:2cm)] B) arc (80:160:2cm);
        
        \node[lvertex] (C) at ([shift=(120:2cm)] B) {};
        \node[lvertex] (D) at ([shift=(120:2cm)] A) {};
        
        \robot{(A)}{0}{1.0}{1}
        \robot{(B)}{-30}{1.0}{2}
        \robot{(C)}{120}{1.0}{3}
        \robot{(D)}{120}{1.0}{4}
        
        \draw[edge] (A)--(B) (B)--(C) (C)--(D) (D)--(A);
    \end{tikzpicture}
    \qquad
    \begin{tikzpicture}[scale=0.7]
        \node[lvertex] (A) at (0,0) {};
        \node[lvertex] (B) at (2,0) {};

        \node[lvertex] (C) at ([shift=(120:2cm)] B) {};
        \node[lvertex] (D) at ([shift=(120:2cm)] A) {};
        
        \robot{(A)}{0}{1.0}{1}
        \robot{(B)}{-30}{1.0}{2}
        \robot{(C)}{120}{1.0}{3}
        \robot{(D)}{120}{1.0}{4}
        
        \draw[edge] (A)--(B) (B)--(C) (C)--(D) (D)--(A) (A)--(C);
    \end{tikzpicture}
    \caption{The realization on the left is flexible: when the robots $3$ and $4$
    slide simultaneously along the indicated circles, the measured distances indicated by edges
    are preserved. The right realization is rigid since the only possible motions preserving the
    measured distances are rotations or translations of the whole formation. The orientation of individual robots does not play any role.}
    \label{fig:rigid_vs_flexible}
\end{figure}

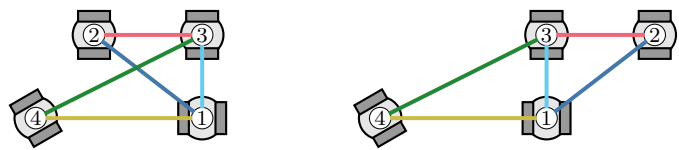
\begin{figure}[h]
    \centering
    \begin{tikzpicture}[scale=1.1]
        \node[lvertex] (A) at (0,0) {};
        \node[lvertex] (B) at (-1.3, 1) {};

        \node[lvertex] (C) at (0,1) {};
        \node[lvertex] (D) at (-2,0) {};
        
        \robot{(A)}{0}{0.64}{1}
        \robot{(B)}{-90}{0.64}{2}
        \robot{(C)}{90}{0.64}{3}
        \robot{(D)}{120}{0.64}{4}
        
        \draw[edge,col7] (A)--(B);
        \draw[edge,col2] (B)--(C);
        \draw[edge,col3] (C)--(D);
        \draw[edge,col4] (D)--(A);
        \draw[edge,col5] (A)--(C);
    \end{tikzpicture}
    \qquad\qquad
    \begin{tikzpicture}[scale=1.1]
        \node[lvertex] (A) at (0,0) {};
        \node[lvertex] (B) at (1.3, 1) {};

        \node[lvertex] (C) at (0,1) {};
        \node[lvertex] (D) at (-2,0) {};
        
        \robot{(A)}{0}{0.64}{1}
        \robot{(B)}{-90}{0.64}{2}
        \robot{(C)}{90}{0.64}{3}
        \robot{(D)}{120}{0.64}{4}
        
        \draw[edge,col7] (A)--(B);
        \draw[edge,col2] (B)--(C);
        \draw[edge,col3] (C)--(D);
        \draw[edge,col4] (D)--(A);
        \draw[edge,col5] (A)--(C);
    \end{tikzpicture}
    \caption{Two configurations of four robots with the same measurement graph and identical measured edge lengths (indicated by colors). Equivalently, they are two non-congruent realizations of the same edge-length-labeled graph. Two additional realizations can be obtained by reflecting each configuration.}
    \label{fig:multiple_realizations}
\end{figure}

However, even if a realization is rigid, the same edge-length data may still admit (finitely) many distinct realizations \emph{up to rigid rotations and translations}; see \Cref{fig:multiple_realizations}.
Determining the number of these realizations has attracted a lot of attention:
from (asymptotic) bounds (e.g.~\cite{BorceaStreinu,BartzosEmirisTzamos,Clarke2025TropicalRigidity}),
via exact numbers for special cases \cite{JacksonOwen2019} and algorithms for computation~\cite{capco2018laman,algSphere,GraseggerHilanyLubbes,DewarGrassegerSchichoTewariWarren2025},
to computational search for graphs maximizing it~\cite{GraseggerKoutschanTsigaridas,capco2024lnumber,grasegger2025explorations}.
For the robots, this means that distance measurements alone need not single out a unique placement of the team; instead, they may determine a \emph{finite set} of feasible formations that all satisfy the same sensing constraints.
This multiplicity is not only a source of ambiguity, but can also be seen as a form of \emph{built-in variability}: the robots may be positioned in different geometric arrangements while preserving the \emph{same set of measured distances}.
For instance, in bar-and-joint linkages (robots as joints, distances as bars), this non-uniqueness has been used to design multistable compliant structures~\cite{ZhangAuzingerBickel}.
Moreover, graph rigidity properties have been applied in various fields beyond multi-robot control
\cite{drones,robots}, including sensor network localization \cite{sensors}, 
molecular biology \cite{biology1,biology2}, and architecture \cite{architecture1}.

In this work, we study the problem of finding \emph{minimally rigid} graphs with high values of selected rigidity invariants, such as the finite number of planar realizations up to rotations and translations. On the one hand, this is an extremal optimization problem over a rapidly growing combinatorial domain (e.g., $\sim\!10^{12}$ minimally rigid graphs with $n=15$ vertices;~\cite{nauty_plugin}), making exhaustive enumeration and evaluation for large $n$ infeasible. On the other hand, improving empirical bounds and understanding the structure of extreme graphs can inform the design of measurement graphs in applications and shed light on the combinatorial and algebraic mechanisms that govern rigidity phenomena. Prior work has relied either on exhaustive search for small $n$~\cite{capco2024lnumber}, or \emph{extensive} computations combined with expert knowledge~\cite{GraseggerKoutschanTsigaridas,grasegger2025explorations}.

In contrast to previous work, we employ reinforcement learning (RL) to search for minimally rigid graphs with exceptionally many realizations.
Concretely, constructing a minimally rigid graph $G_n$ on $n$ vertices can be formulated as a sequential decision-making problem
using so called $0$- and $1$-\emph{extensions} (introduced by~\cite{Henneberg})
that add a new vertex each time.
At step $k$, the current intermediate rigid graph $G_k$ defines a \textit{state} $s_k \in \mathcal{S}$; each admissible extension of $G_k$ that yields a rigid graph $G_{k+1}$ corresponds to an \textit{action} $a_k \in \mathcal{A}$; and the value of a chosen rigidity invariant (e.g., the number of realizations of the terminal graph~$G_n$) serves as the \textit{reward}.
The objective is to learn a policy $\pi$ that, given a state $s_k$, selects actions $a_k$ that transform $G_k$ into $G_{k+1}$, and repeats this process until reaching a final graph $G_n$ with high reward. We address this problem via the Deep Cross-Entropy Method (Deep CEM), which iteratively updates a stochastic policy over graph extensions by biasing sampling toward elite construction trajectories.
To represent intermediate graphs, we employ a Graph Isomorphism Network (GIN) encoder, and we design an extension-level, permutation-equivariant policy module so that action probabilities transform consistently under vertex relabeling. Our main contributions can be summarized as follows:

\begin{itemize}
    \item  We formulate the construction of minimally rigid graphs with large realization counts as a sequential decision-making problem and, to our best knowledge, we are the first to propose a RL framework\footnote{See \url{https://github.com/Esestree/RL-min-rigid-graphs}} with permutation-equivariant action representations for its solution.

    \item We demonstrate through extensive empirical study that our approach matches the known optima in the planar case and improves the best known lower bounds on the sphere for larger $n$, including \textbf{new record graphs} for spherical realization counts. 
    
    \item Of particular relevance, our method \emph{is orders of magnitude faster} than previous attempts and yields extremal graph instances that can potentially provide evidence toward a better understanding of structural phenomena in rigidity theory, which we further discuss (see Results in \Cref{sec:experiments} and \Cref{sec:extension_impact}).

    \item We study the  transfer-learning properties of our method (\Cref{sec:transfer_learning}) and ablate our encoder (\Cref{sec:ablation}).
\end{itemize}

\section{Preliminaries and Problem Definition}

\paragraph{Basic Notation.}
We consider only simple undirected graphs.
Let $G = (V, E)$ be a graph with $n = |V|$ vertices and $|E|$ edges.
We denote the edge formed by vertices $u$ and~$v$ by $uv\in E$.
Since the edges are undirected, $uv\in E$ is the same as $vu \in E$.
We denote by~$K_n$ the clique on $n$ vertices. A summary of the notation is provided in the appendix in \Cref{tab:notation}.

\subsection{Minimally Rigid Graphs}

We briefly recall the definition of \emph{minimally rigid graphs} for completeness and to fix notation. To this end, we first introduce flexible and rigid realizations, and then define generic rigidity and minimal rigidity. For a comprehensive treatment, we refer to~\cite[Part~IV]{handbook}.

\begin{definition}[Flexible and Rigid Realizations]
  A \emph{planar realization} of a graph $G=(V,E)$ is a map $p\colon V \rightarrow \RR^2$.
  It is called \emph{flexible}
  if it has a nontrivial \emph{flex}, that is, a continuous path of planar realizations
  $p_t$, $t \in [0, 1]$, with $p_0 = p$ such that for all $t \in (0, 1]$,
  $||p(u)-p(v)||=||p_t(u)-p_t(v)||$ holds for every edge $uv\in E$, but $||p(u)-p(v)||\neq ||p_t(u)-p_t(v)||$ for some vertices $u,v\in V$.
  Otherwise, the realization is called \emph{rigid}.
\end{definition}

In particular, for flexible realizations, the requirement that the flex is nontrivial rules out translations and rotations of the entire framework and ensures that the deformation changes the shape of the realization.

\begin{definition}[Generic and Minimal Rigidity]
  A planar realization is \emph{generic} if the coordinates of the vertices are algebraically independent over the rational numbers.
  A graph~$G$ is \emph{rigid} if any (equivalently, every) generic planar realization is rigid.
  Otherwise, the graph is called \emph{flexible}.
  A graph is \emph{minimally rigid} if it is rigid and the deletion of any edge yields a graph that is flexible.
\end{definition}

\subsection{Rigidity Invariants}

We study three rigidity invariants for minimally rigid graphs. 
Given a minimally rigid graph $G$, we consider:
(i)~the \emph{number of (complex) planar realizations} of $G$;
(ii)~the \emph{number of (complex) spherical realizations} of $G$; and
(iii)~the \emph{number of NAC-colorings} of $G$.

Given a generic planar realization $p$ of a minimally rigid graph $G=(V,E)$, we aim to count the realizations of $G$ that have the same edge lengths as~$p$, modulo rotations and translations. Note that this number is finite,
but it might vary between different choices of $p$ (see for instance~\cite{JacksonOwen2019}) as this corresponds to the number of solutions of the following polynomial system over the real numbers:
\begin{align}
  \label{eq:distances}
  (x_u - x_v)^2 + (y_u - y_v)^2 &= ||p(u)-p(v)||^2 \quad \forall uv \in E\,, \nonumber \\
  x_{\bar{u}} = y_{\bar{u}} = y_{\bar{v}} = 0\,, \quad x_{\bar{v}} &= ||p(\bar{u})-p(\bar{v})|| \,,
\end{align}
where $(x_u, y_u)$ denotes the coordinates of vertex $u\in V$ and
some edge $\bar{u}\bar{v}\in E$ is pinned down to factor out the rotations and translations.
However, the number of solutions over the complex numbers is the same for all generic realizations $p$
of the graph $G$~\cite[Theorem~3.6]{JacksonOwen2019} and upper bounds the number of (real) realizations.

\begin{definition}[The number of planar realizations]
  Let $G$ be a minimally rigid graph.
  The \emph{number of (complex) planar realizations} of $G$, denoted by $\numPlaneFun{G}$, is the number of complex solutions
  of \eqref{eq:distances} for any generic realization $p$ of $G$.
\end{definition}

Similarly as we have defined planar realizations, \emph{spherical} realizations and their rigidity
can be defined by replacing~$\RR^2$ by the $2$-dimensional sphere.
The generic rigidity behavior in the plane and the sphere is the same (see for instance \cite[Theorem~2.5]{DewarGrasegger}), namely,
a graph is minimally rigid in the plane if and only if it is minimally rigid on the sphere.
Hence, one may consider rigid graphs without specifying whether rigidity is studied in the plane or on the sphere.
A polynomial system of equations similar to $\eqref{eq:distances}$ corresponding
to spherical realizations can be formulated using inner products on the left hand side.
The number of complex solutions of the system is then called the \emph{number of spherical realizations} of a minimally rigid graph $G$
and is denoted by $\numSphereFun{G}$. See \cite[Section~4]{DewarGrasegger} for a rigorous definition.

Finally, we consider a third rigidity-related invariant ($\numNACinv$, also defined for flexible graphs), defined via certain two-color edge assignments as follows.

\begin{definition}[NAC-colorings]
  A surjective edge coloring of a graph $G$ by red and blue
  is called a \emph{NAC-coloring} if for every cycle of $G$, either all edges have the same color,
  or there are at least two blue and two red edges.
  Let $\numNACfun{G}$ denote the number of NAC-colorings of~$G$ divided by two (which corresponds to swapping the colors).
\end{definition}
NAC-colorings are interesting from the rigidity point of view, since they yield flexible realizations. More precisely, a connected graph has a planar realization with a flex if and only if it has a NAC-coloring~\cite{GraseggerLegerskySchicho2019}.
Interestingly, rigid graphs can have (non-generic) realizations with a flex. Different NAC-colorings give different flexes (see~\cite[Section~2.2]{clinch2024nac}), hence it is natural to ask for minimally rigid graphs with many NAC-colorings.

\subsection{Optimization Problem}

Fix $n \ge 2$ and let $\nGraphs$ be the set of minimally rigid graphs on $n$ vertices.
For a chosen rigidity invariant $\rigInv:\nGraphs \to \nonnegReals$ (e.g., $\rigInv=\numPlaneInv$, $\rigInv=\numSphereInv$, or $\rigInv=\numNACinv$),
we consider the following optimization problem:
\begin{equation}
\label{eq:combinatorial_obj}
G^\star \in \arg\max_{G \in \nGraphs} \rigInv(G).
\end{equation}
Since $\nGraphs$ grows super-exponentially (see \Cref{sec:superexp}) and evaluating $\rigInv(G)$ can be expensive, directly solving~\eqref{eq:combinatorial_obj}
is infeasible beyond small~$n$. Thus, we pursue an approximate solution via RL, by learning a policy that constructs high-reward graphs through sequential extensions.

\section{Methodology}

\subsection{Deep CEM for Minimally Rigid Graphs}

In our task of generating minimally rigid graphs with high realization counts, we propose a framework based on the Deep CEM. Classical CEM~\cite{rubinstein2004cross} iteratively samples candidates from a simple parametric distribution (e.g., a Gaussian) and updates the distribution to concentrate probability mass around the highest-performing (elite) solutions. Deep CEM extends this idea by replacing the simple parametric distribution with a \emph{stochastic policy network}~$\pi_\theta$, which can represent complex and structured distributions (e.g., over graph extensions). 
A \emph{stochastic policy network} with parameters $\theta$ can be seen as a mapping
\[
\pi_\theta: \mathcal{S} \to \Delta(\mathcal{A}),
\]
where $\mathcal{S}$ is the state space, $\mathcal{A}$ is the action space, and $\Delta(\mathcal{A})$ denotes the set of probability distributions over $\mathcal{A}$.  
For each state $s \in \mathcal{S}$, $\pi_\theta(\cdot \mid s)$ is a probability distribution over actions in that state, satisfying
$
\sum_{a \in \mathcal{A}} \pi_\theta(a \mid s) = 1.
$
Thus, the value $\pi_\theta(a \mid s)$ gives the probability of selecting action $a$ in state $s$.

In our setting, a state $s$ corresponds to an intermediate graph~$G_k$ in the construction process. Rolling out the  policy~$\pi_{\theta_t}$ (at step $t$) from the base graph $K_2$ (the clique on two vertices) many times produces construction sequences
\[
K_2 = G_2 \to G_3 \to \cdots \to G_n,
\]
where each transition $G_k \to G_{k+1}$ is obtained by sampling an extension from ${\pi_{\theta_t}(\cdot \mid G_k)}$. Extensions are particularly natural choices to define the action space  $\mathcal{A}$, because a graph~$G$ is minimally rigid if and \emph{only} if it can be obtained from $K_2$ by a finite sequence of $0$- and $1$-extensions (see e.g.~\cite[Theorems~19.2 and 19.11]{handbook}), whose definition we recall now.
Let $G=(V,E)$ be a graph, and let $z \notin V$ be a new vertex. If $u,v \in V$ are distinct vertices, then the graph
\[
    (V \cup \{z\},\, E \cup \{uz, vz\})
\]
is called a \emph{$0$-extension} of $G$.
If $u,v,w \in V$ are distinct vertices and $vw \in E$, then the graph
\[
    (V \cup \{z\},\, E \setminus \{vw\} \cup \{uz, vz, wz\})
\]
is called a \emph{$1$-extension} of $G$; for illustration, see Figure~\ref{fig:extensions}. Thus, for a transition from~$G_k$ to $G_{k+1}$, we define the action space $\mathcal{A}$ as the set of all $0$-extensions and $1$-extensions applicable to $G_k$.

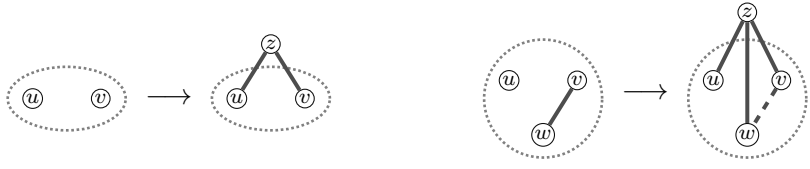
\begin{figure}[ht]
\centering
\begin{tikzpicture}[scale=0.6]
    %--- 0-extension (left)
    \node[lvertex] (u0) at (0,0) {\scriptsize $u$};
    \node[lvertex] (v0) at (1.5,0) {\scriptsize $v$};
    \draw[graphIndication] (0.75,0) ellipse (1.3 and 0.7);

    % arrow
    \node[draw=none, fill=none] at (3,0) {$\longrightarrow$};

    %--- 0-extension (right)
    \node[lvertex] (u1) at (4.5,0) {\scriptsize $u$};
    \node[lvertex] (v1) at (6,0) {\scriptsize $v$};
    \node[lvertex] (w1) at (5.25,1.2) {\scriptsize $z$};
    \draw[graphIndication] (5.25,0) ellipse (1.3 and 0.7);
    \draw[edge] (w1)--(u1);
    \draw[edge] (w1)--(v1);

    \begin{scope}[xshift=10.5cm, yshift=3.4cm]
        %--- 1-extension (left)
        \node[lvertex] (u2) at (0,-3) {\scriptsize $u$};
        \node[lvertex] (v2) at (1.5,-3) {\scriptsize $v$};
        \node[lvertex] (w2) at (0.75,-4.2) {\scriptsize $w$};
        \draw[edge] (v2)--(w2);
        \draw[graphIndication] (0.75,-3.4) circle (1.3);
    
        % arrow
        \node[draw=none, fill=none] at (3,-3.3) {$\longrightarrow$};
    
        %--- 1-extension (right)
        \node[lvertex] (u3) at (4.5,-3) {\scriptsize $u$};
        \node[lvertex] (v3) at (6,-3) {\scriptsize $v$};
        \node[lvertex] (w3) at (5.25,-4.2) {\scriptsize $w$};
        \node[lvertex] (z3) at (5.25,-1.5) {\scriptsize $z$};
        \draw[graphIndication] (5.25,-3.4) circle (1.3);
        \draw[edge] (z3)--(u3);
        \draw[edge] (z3)--(v3);
        \draw[edge] (z3)--(w3);
        \draw[edge, dashed] (v3)--(w3);
    \end{scope}
\end{tikzpicture}
\caption{Illustration of the 0-extension (left) and 1-extension (right).
In the 1-extension, the edge $vw$ is removed (dashed).}
\label{fig:extensions}
\end{figure}

Having fixed the reward function $\rigInv\colon \nGraphs\rightarrow \nonnegReals$ for a given rigidity invariant (e.g., $\rigInv = \numPlaneInv$), we now describe how Deep CEM is instantiated in our setting; see Algorithm~\ref{alg:deep_cem_reduced}. 
For a given number of vertices $n$, and number of generations $T$, we maintain a population of $m$ constructions and iteratively refine a stochastic policy $\pi_\theta$ over $0$- and $1$-extensions. At generation~$t$, we start from the current survivor set $S_{t-1}$ and build a population~$P_t$. 
Each new candidate is obtained by rolling out the policy from the base graph~$K_2$: at step $k$ we sample an extension $a_k \sim \pi_{\theta_t}(\cdot \mid G_k)$, apply it to obtain~$G_{k+1}$, and repeat until we reach $G_n$. Because the action space consists only of $0$- and $1$-extensions, every sampled $G_n$ is minimally rigid by construction.

We then evaluate the reward $\rigInv(G)$ for all $G \in P_t$ and select an elite set $\Omega_t$ comprising the top $\lfloor m \rho_{\text{elite}}\rfloor$ graphs. From each $G \in \Omega_t$ we extract the sequence of state--action pairs $\{(G_k,a_k)\}_{k=2}^{n-1}$ and aggregate them into a dataset $\mathcal{D}_t$. Next, we update the policy parameters from $\theta_t$ to $\theta_{t+1}$ by minimizing for few epochs $\mathcal{L}(\theta_t)$ (Eq.~\ref{eq:entropy}) on $\mathcal{D}_t$ using Adam. The survivor set $S_t$ for the next generation is then taken as the top $\lfloor m \rho_{\text{surv}}\rfloor$ graphs in~$P_t$. After $T$ generations, the algorithm outputs
\[
\hat{G}^\star \gets \arg\max_{G \in S_T} \rigInv(G),
\]
an $n$-vertex minimally rigid graph with a large value of $\rigInv(G)$.

\begin{algorithm}
\caption{Deep CEM for Minimally Rigid Graphs}
\label{alg:deep_cem_reduced}
\begin{algorithmic}[1]
\Require number of vertices $n$; population size $m$; number of generations $T$;
survivor fraction $\rho_{\text{surv}}$; elite fraction $\rho_{\text{elite}}$;
reward function $\rigInv\colon \nGraphs \rightarrow \nonnegReals$
\State  $S_0 \gets \emptyset$ ; $\theta_1 \gets $ initialize policy parameters
\For{$t = 1$ to $T$}  \Comment{{for details, see \Cref{sec:rep}}}
    \State $P_t \gets S_{t-1}$ \Comment{current population of constructions}
    \While{$|P_t| < m$}
        \State $G_2 \gets K_2$
        \For{$k = 2$ to $n-1$}
            \State sample extension $a_k \sim \pi_{\theta_t}(\cdot \mid G_k)$
            \State $G_{k+1} \gets \mathrm{Apply}(G_k, a_k)$
        \EndFor
        \State add $G_n$ to $P_t$
    \EndWhile
    \State $\Omega_t \gets$ top $\lfloor m \rho_{\text{elite}}\rfloor$ graphs in $P_t$ by $\rigInv(G)$
    \State build dataset $\mathcal{D} = \bigcup_{G \in \Omega} \{(G_k, a_k)\}_{k=2}^{n-1}$
    %\State $\theta_{t+1} \gets$ gradient descent (Adam) on $\mathcal{L}(\theta_t)$ using $\mathcal{D}$
    \State $\theta_{t+1} \gets \mathrm{Update}(\theta_t;\mathcal{D},\mathcal{L})$ 
    \State $S_t \gets$ top $\lfloor m \rho_{\text{surv}}\rfloor$ graphs in $P_t$ by $\rigInv(G)$
\EndFor
\State \Return $\hat{G}^\star \gets \arg\max_{G \in S_t} \rigInv(G)$
\end{algorithmic}
\end{algorithm}

For computing $\mathcal{R} \in\{\numPlaneInv,\numSphereInv\}$,
we use the package \textsc{lnumber}~\cite{capco2024lnumber} which is based on~\cite{capco2018laman}
and~\cite{algSphere} respectively.
For $\numNACinv$, we use the algorithm~\cite{lastovicka2024nac} implemented in the package \textsc{PyRigi}~\cite{pyrigi}.

\paragraph{Remark.} During graph generation, we must score thousands of candidates to form the elite set $\Omega_t$ in each generation~$t$~(see Algorithm~\ref{alg:deep_cem_reduced}, line~12). Since exact evaluation of $\numPlaneInv$ and $\numSphereInv$ is costly (a prevailing view in rigidity theory community is that determining the number of realizations is inherently difficult and intrinsic to the nature of the problem, even if this has not been formally proved yet), we use the efficiently computable upper bound $\text{m-Bézout}(G)$~\cite{bartzos2020bezout} as a surrogate during sampling. In particular, $\text{m-Bézout}(G)$ upper-bounds the realization-count objectives, i.e.,
\[
\numPlaneInv(G) \leq \numSphereInv(G) \leq  \text{m-Bézout}(G)
\]
(see~\cite{DewarGrasegger} for the first inequality).
In practice, we first score all graphs in a generation with $\text{m-Bézout}(G)$ and then compute the true reward ($\numPlaneInv$ or $\numSphereInv$) only for the top $25\%$ candidates. This two-stage screening yields a significant speedup in our experiments while maintaining good agreement with the exact reward on the selected candidates. For details, we refer to \Cref{sec:surrogate}.

\subsection{Stochastic Policy Network Architecture}

In this section, we describe the architecture of our stochastic policy network, which is illustrated in Figure~\ref{fig:policy}. 

\paragraph{Graph Encoder.} The input to the policy at step $k$ is an intermediate (minimally rigid) graph $G_k = (V_k, E_k)$. Since the choice of extensions must depend only on the underlying unlabeled combinatorial structure, we require the policy to be permutation-equivariant: relabeling the vertices of an intermediate graph $G_k$ with permutation $\sigma$ permutes the policy’s output in the corresponding way:
\begin{align}
\pi_\theta(\cdot \mid \sigma(G_k)) = \sigma\big(\pi_\theta(\cdot \mid G_k)\big)
\label{eq:equiv}
\end{align}
Moreover, the rigidity invariants we study are themselves invariant under vertex relabeling.

Thus, a natural graph neural network encoder for this setting is the Graph Isomorphism Network (GIN)~\cite{xu2019gin}. GIN not only discriminates between different graph structures but also maps similar structures to similar embeddings, capturing dependencies between graph components. At a high level (we refer to the original work for details), a GIN layer aggregates features from a vertex’s neighbors, combines them with the vertex’s own features, and applies a learned transformation.
Formally, the hidden representation~$h_v^{(l)}$ of vertex $v$ at layer $l \in \{1,2,\dots, L_\mathcal{G}\}$ is given by
\[
h_v^{(l)} = \mathrm{MLP}^{(l)}\!\Bigl(\bigl(1 + \epsilon^{(l)}\bigr)\, h_v^{(l-1)} \;+\; \sum_{u \in \mathcal{N}(v)} h_u^{(l-1)}\Bigr),
\]
where $\mathcal{N}(v)$ denotes the neighbors of $v$, $\epsilon^{(l)}$ is a learnable scalar and $\mathrm{MLP}^{(l)}$ denotes a multilayer perceptron. The initial vertex features $h_v^{(0)}$ encode basic local structural information around each vertex and serve as the input to the first GIN layer.
Note that GIN may fail to capture certain substructures and to distinguish between specific classes of graphs (e.g., regular graphs, certain trees or other graphs that share the same degree multiset) when all vertices share the same feature values, because they were shown to be at most as powerful as the $1$-Weisfeiler–Lehman ($1$-WL) graph isomorphism test~\cite{huang2022wl}. 

\begin{figure}[t]
    \centering
    \includegraphics[width=\linewidth]{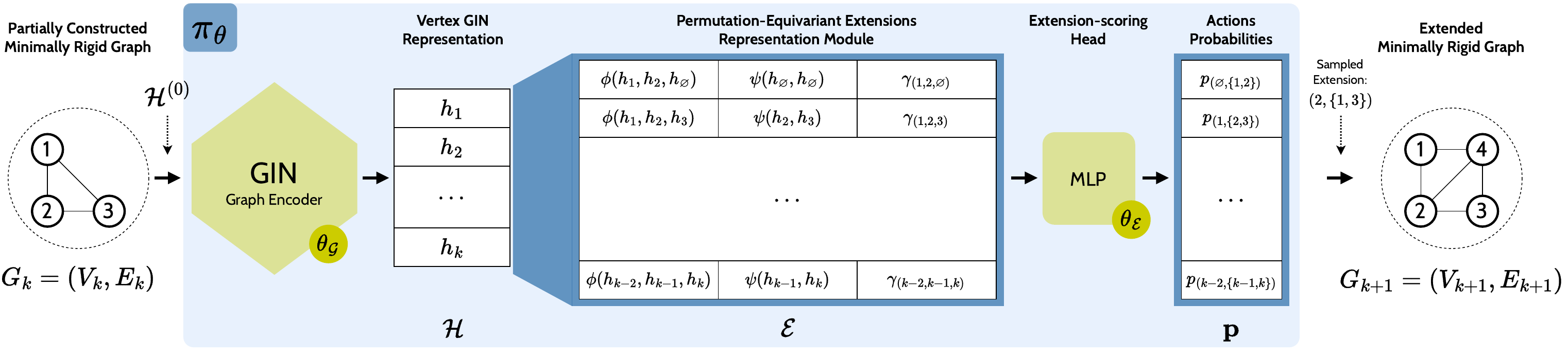}
    \caption{Policy network architecture. Given an intermediate graph $G_k$, a GIN encoder produces vertex embeddings $\mathcal{H}$. 
    For each candidate $0$- or $1$-extension (defined by a vertex tuple), we construct an extension-level representation in $\mathcal{E}$ by combining the relevant vertex embeddings with extension-specific indicators (e.g., validity and move type). A shared extension-scoring head outputs a logit per candidate extension, and a softmax normalizes these logits to obtain action probabilities. In the figure, the learnable parameters of the policy $\pi_{\theta}$ are $\theta = (\theta_G,\theta_{\mathcal E})$, where $\theta_G$ parameterizes the GIN encoder and $\theta_{\mathcal E}$ parameterizes the extension-scoring head. The policy is permutation-equivariant: any relabeling of vertices induces the corresponding permutation of extension probabilities.}

    \label{fig:policy}
\end{figure}

To better expose rigidity-relevant structure to the encoder, we therefore augment vertices with explicit local structural features. In particular, prior work on rigidity invariants~\cite{grasegger2025explorations} observed recurring patterns in graphs with high $\numPlaneInv$ or $\numSphereInv$, including characteristic degree multisets, constraints on triangle counts, or degree- and Hamiltonicity-driven adjacency rules. These observations indicate that degree-related structural properties are highly relevant for characterizing graphs with favorable rigidity invariants.  Thus, for each vertex $v \in V_k$, we build a Local Degree Profile (LDP)~\cite{errica2018baseline} defined as
\begin{align*}
\LDP{v} = \Bigl(&
\deg(v), \;
\min_{u \in \mathcal{N}(v)} \deg(u), \;
\max_{u \in \mathcal{N}(v)} \deg(u), \; 
\mathop{\mathrm{mean}}\limits_{u \in \mathcal{N}(v)} \deg(u), \;
\mathop{\mathrm{std}}\limits_{u \in \mathcal{N}(v)} \deg(u)
\Bigr),    
\end{align*}
where $\deg(v)$ is the degree of vertex $v$. In summary, we define the initial vertex features as
\[
h_v^{(0)} = [\LDP{v} \,;\, \lambda_k \,;\, \kappa_v],
\]
where $\lambda_k \in \RR^{d}$ is a learnable embedding of the current construction step $k$ (making the policy step-aware), and $\kappa_v$ is the clustering coefficient of $v$, defined for unweighted graphs as the fraction of potential triangles through $v$ that are present:
\[
    \kappa_v = \frac{2\tau(v)}{\deg(v)\bigl(\deg(v) - 1\bigr)},
\]
with $\tau(v)$ denoting the number of triangles incident to $v$. Finally, we define the initial vertex features as $\mathcal{H}^{(0)} = (h_v^{(0)})_{v \in V_k }$ and denote by
$h_v = h_v^{(L_\mathcal{G})}$ the encoded representation of vertex $v$, i.e., the output of the final GIN layer.

\paragraph{Permutation-Equivariant Extensions Representation Module.} We first pass the initial vertex features through our encoder (GIN layers) and obtain vertex representations $\mathcal{H}=(h_v)_{v \in V_k }$.
Our goal is then, based on $\mathcal{H}$, to construct an architecture that assigns probabilities to each extension. Note that the extensions are defined in terms of specific vertex tuples (e.g., a $1$-extension on $(u,v,w)$); therefore, under a relabeling of vertices, the \emph{set} of valid extensions is
permuted accordingly. Consequently, for isomorphic graphs~$G$ and $G'$, the desired behavior is that the policy on $G'$ is a \emph{permutation} of the policy on $G$ (see~\eqref{eq:equiv}).

Thus, we must construct extension-level representations $\mathcal{E}$ from vertex embeddings~$\mathcal{H}$ so that action scores transform equivariantly under graph isomorphisms. Observe that, in an intermediate graph $G_k$, a candidate $1$-extension acting on vertices $u,v,w$ with removal of edge $vw$ can be encoded by~$(u,\{v,w\})$. For all distinct $u,v,w\in V_k$,
we consider
\[
e_{(u,\{v,w\})} = \bigl[\phi(h_u, h_v, h_w) \,;\, \psi(h_v, h_w) \,;\, \gamma_{(u,v,w)}\bigr] \in \mathcal{E},
\]
where
$\phi(\cdot)$ and $\psi(\cdot)$ are permutation-invariant functions (see~\cite{zaheer2017deepsets}), and $\gamma_{(u,v,w)}$ is a feature vector encoding additional properties of the extension.
If $vw\notin E_k$, we say that $(u, \{v,w\})$ is invalid,
otherwise $e_{(u,\{v,w\})}$ represents a $1$-extension.
Note that $0$-extensions can be modeled as a special case of the above construction by introducing a dummy vertex $\emptyset$ with embedding $h_{\emptyset} = \vec{0}$. A $0$-extension acting on vertices $u$ and $v$ is then represented as
\[
e_{(\emptyset,\{u,v\})} = \bigl[\phi(h_u, h_v, h_{\emptyset}) \,;\, \psi(h_{\emptyset}, h_{\emptyset}) \,;\, \gamma_{(u,v,\emptyset)}\bigr] \in \mathcal{E},
\]
so that both $0$- and $1$-extensions share the same extension representation space $\mathcal{E}$. Furthermore, $\gamma_{(u,v,w)} \in \{0,1\}^5$ collects five extension-specific features  
\[
\gamma_{(u,v,w)} = \bigl[
\mathbf{1}_{\text{invalid}}, 
\mathbf{1}_{E_0}, 
\mathbf{1}_{E_{1a}},
\mathbf{1}_{E_{1b}}, 
\mathbf{1}_{E_{1c}}
\bigr].
\]
where each component is a binary indicator:
$\mathbf{1}_{\text{invalid}}$ denotes whether $(u,\{v,w\})$ is invalid, 
$\mathbf{1}_{E_0}$ indicates a $0$-extension, and 
$\mathbf{1}_{E_{1a}}$, $\mathbf{1}_{E_{1b}}$, $\mathbf{1}_{E_{1c}}$ indicate the $1$-extension subclasses identified in~\cite{grasegger2025explorations}, which differ only in the number of edges connecting the three vertices.
Exactly one of the features $\mathbf{1}_{\text{invalid}},\mathbf{1}_{E_0},\mathbf{1}_{E_{1a}},\mathbf{1}_{E_{1b}},\mathbf{1}_{E_{1c}}$ is equal to $1$.

\paragraph{Remark.}
We encode extension validity via the binary feature $\mathbf{1}_{\text{invalid}}$ instead of removing invalid moves. This keeps a uniform extension representation and lets a shared extension-scoring MLP process all candidates, while learning to assign (near-) zero probability to those with $\mathbf{1}_{\text{invalid}}=1$. This soft treatment stabilizes training (invalid moves still receive gradients) and preserves permutation structure. In the rare case of an invalid extension is sampled, we simply discard it and resample, so the construction remains minimally rigid.

\paragraph{Extension-scoring Head.} Finally, we feed each $e_i \in \mathcal{E}$ through an $L_\mathcal{E}$-layer MLP, applied element-wise with shared parameters, so that the final layer produces a scalar logit $z_i$ for each extension while preserving equivariance. Collecting these into a vector $z \in \mathbb{R}^{|\mathcal{E}|}$ and applying a softmax,
\[
p_{(u,\{v,w\})} \;=\; \frac{\exp(z_{(u,\{v,w\})})}{\sum_{a \in V_k \cup \{\emptyset\}; b,c \in V_k} \exp(z_{(a,\{b,c\})})},
\]
yields a probability distribution $\mathbf{p}$ over candidate extensions, where $p_{(u,\{v,w\})}$ is the probability of selecting the extension indexed by $(u,\{v,w\})$ with $u\in V_k \cup \{\emptyset\}$ and $v, w \in V_k$.

\subsection{Loss Function}

A central challenge in reinforcement learning is achieving a suitable trade-off between \emph{exploration} and \emph{exploitation}. In our setting, this is particularly critical: we search for \emph{extremely rare} graph configurations, and insufficient exploration can cause the policy to converge prematurely to suboptimal structures. To explicitly address this requirement, we adopted Entropy regularization. 
This regularization technique augments the training objective (negative log-likelihood of elite actions) with an entropy term that penalizes overly confident action distributions:
\begin{align}
\mathcal{L}(\theta) = - \frac{1}{|\mathcal{D}|}\sum_{(G_j, a_j) \in \mathcal{D}} \log \pi_\theta(a_j \mid G_j) 
                    \;-\; \eta \,\frac{1}{|\mathcal{D}|}\sum_{(G_j, a_j) \in \mathcal{D}} H\bigl(\pi_\theta(\cdot  \mid G_j)\bigr),
\label{eq:entropy}
\end{align}
where $H\bigl(\pi_\theta(\cdot  \mid s)\bigr) = -\sum_{a \in \mathcal{A}} \pi_\theta(a \mid s) \log \pi_\theta(a \mid s)$ is the entropy of the action distribution at state $s$,
and $\eta\in\nonnegReals$ controls the strength of entropy regularization. As training progresses, the model acquires more information about which actions are beneficial. If entropy is kept too high, the policy may remain overly stochastic and fail to reliably construct high-quality graphs. To balance this, we gradually decrease the entropy coefficient over generations:
\begin{equation}
    \eta_t = \frac{\eta_0}{1 + \alpha \ln\!\left(1 + t \, e^{-\beta}\right)},
    \label{eq:ent_coef_decay}
\end{equation}
where $\eta_0$, $\alpha$, and $\beta$ are hyperparameters, and $t \leq T$ indexes the generation. This schedule encourages broad exploration in the early stages and progressively shifts the policy toward exploitation of the most promising graph constructions.
Although other methods can be similarly efficient in practice, entropy regularization has some unique properties. First, because it modifies the objective, it encourages the network to build good constructions, while being as uncertain in action choices as possible, which helps in creating diverse generations of good constructions. Second, the network learns to concentrate its `confidence budget' on terminal actions, optimizing the final steps to maximize rewards even when preceding actions were chosen randomly.

\section{Experiments}
\label{sec:experiments}

We evaluate our method on  $\numPlaneInv$, $\numSphereInv$, and $\numNACinv$. We focus on the range $10\leq n \leq 18$, matching the intervals most extensively studied in prior computational work and for which reference optima or best-known bounds are available. 
For small $n$, exhaustive enumeration provides ground-truth optima, whereas for larger $n$ we compare against the strongest previously reported values. For further details on reproducibility (including hardware specifications) and hyperparameter choices, we refer to \Cref{sec:rep} and the accompanying repository \cite{project-repo}.

\paragraph{Baselines and Metrics.} We compare our method in terms of (i) the best realization counts reported in the literature and (ii) runtime. For the number of realizations,
we compare against \cite{clinch2024nac,grasegger2025explorations},
which includes also the results obtained previously by~\cite{GraseggerKoutschanTsigaridas,capco2024lnumber}.
The timing for the results from~\cite{grasegger2025explorations} were obtained as follows:
for the results known to be optimal,
we measured the exhaustive search using~\cite{capco2024lnumber} 
except for the largest number of vertices, where the timing was extrapolated using the numbers of non-isomorphic minimally rigid graphs (from~\cite{nauty_plugin}).
Since~\cite{grasegger2025explorations} claims that \emph{several millions} of graphs with $n \geq 15$ were checked to get the indicated bounds on $\numPlaneInv$, without providing any timing,
we estimate the time of computing $\numPlaneInv$ for \emph{one million}  of graphs with 15 vertices
(using the average time to compute $\numPlaneFun{G}$ for $15$-vertex graphs $G$ we have encountered), similarly with half milion $14$-vertex graphs for $\numSphereInv$. Note this is a (very) conservative estimate; in reality it may be (much) higher.
For our model, to initialize parameters for the computation of $\numPlaneInv$ on $n=18$ vertices, 
we used the weights of the model trained for $n-1$ vertices, and similarly for $\numSphereInv$ for $n\in\{16, 18\}$; hence, we include this pretraining in the timing.

\paragraph{Results.} The results we obtained are summarized in \Cref{tab:results}.
Besides the timings, we indicate also for how many graphs the rigidity invariants were evaluated.
Remarkably, the number of graphs needed for training is by several orders of magnitude lower
than the number of minimally rigid graphs. For instance, for $n=15$, we evaluate $\numPlaneInv$ and $\numSphereInv$
only for less than $10^5$ out of the almost $10^{12}$ graphs.
Empirically, the number of minimally rigid graphs on $n+1$ vertices is expected to be at least $30\times$ larger than on $n$. However, we observe that the number of graph evaluations required by our method to find constructions with high rigidity invariants increases only mildly as $n$ grows. Thus, the longer computation times are caused by the (costly) reward evaluation for larger graphs, especially for $\numSphereInv$, rather than by the number of evaluated graphs (see~$N$ in \Cref{tab:results}).

For $\numPlaneInv$, we found exactly the same (i.e., isomorphic) graphs and therefore also the same values as~\cite{grasegger2025explorations}.
However, for larger $n$, our method finds them orders of magnitude faster than exhaustive search.
Note that the implementation~\cite{capco2024lnumber} is already optimized,
for instance, it skips graphs with a vertex of degree two as it is known these cannot attain the maximum.

\newgeometry{
  top=3cm,
  bottom=3cm,
}
\begin{landscape}
\thispagestyle{empty}
\begin{table}
\centering
{\footnotesize
\begin{tabular}{@{}ll@{}cccccccccc}
\toprule
& \# vertices $n$ &  10 & 11 & 12 & 13 & 14 & 15 & 16 & 17 & 18\\
& \# min. rigid graphs & 110 K & 2 M & 44 M & 1 G & 30 G & 932 G & - & - & -  \\
\midrule

\multirow{5}{*}{\rotatebox[origin=c]{90}{$\numPlaneInv$}} &
\cite{grasegger2025explorations} & {880}$^\star$ & {2\,288}$^\star$ & {6\,180}$^\star$ & {15\,536}$^\star$ & {42\,780}$^\star$ & 112\,752 & 312\,636 & 877\,960 & 2\,414\,388\\
& Time    & 1 s & 28 s & 1177 s & 17 h & \textcolor{gray}{510 h} & \textcolor{gray}{36 h} & - & - & - \\
\cmidrule(rl){2-11} 
& Ours  & 880$^\star$ & 2\,288$^\star$ & 6\,180$^\star$ & 15\,536$^\star$   & 42\,780$^\star$ & 112\,752 & 312\,636 & 877\,960 & 2\,414\,388 \\
& Time (Hit / Total)  & 42 s / -- & 66 s  / -- & 387 s  / -- & 137 s  / -- & 3208 s   / -- & 3 h  / 6 h & 8 h / 21 h & 33 h / 50 h & 83 h / 105 h\\
& $N$ (Hit / Total) & 2 K / --  & 2 K / --  & 9 K / --  & 3 K / --  & 29 K  / --  & 41 K / 76 K& 38 K / 87 K & 44 K / 62 K & 68 K / 72 K \\
\midrule
% sphere
\multirow{5}{*}{\rotatebox[origin=c]{90}{$\numSphereInv$}} &
\cite{grasegger2025explorations} & 1\,536$^\star$ & 4\,352$^\star$ & 12\,288$^\star$ & 34\,816$^\star$ & 98\,304 & 274\,432 & 815\,104 & 2\,195\,456 & -\\
& Time & 7 s & 470 s & 10 h & \textcolor{gray}{227 h} & \textcolor{gray}{158 h} & - & - & - & - \\
\cmidrule(rl){2-11} 
& Ours & 1\,536$^\star$ & 4\,352$^\star$ & 12\,288$^\star$ & 34\,816$^\star$ & 98\,304 & \textbf{278\,528} & \textbf{819\,200} & \textbf{2\,228\,224} & \textbf{6\,127\,616}\\
& Time (Hit / Total) & 10 s / -- & 89 s / -- & 1209 s / -- & 3 h / -- & 24 h / 41 h & 41 h / 139 h & 401 h / 446 h & 634 h / 890 h & 923 h / 1026 h\\
& $N$ (Hit / Total) & 1 K / --  & 3 K / -- & 13 K / -- & 34 K / -- & 70 K / 111 K & 24 K / 86 K & 123 K / 130 K & 12 K / 17 K & 17 K / 18 K\\
\midrule
% NAC
\multirow{5}{*}{\rotatebox[origin=c]{90}{$\numNACinv$}} &
\cite{clinch2024nac} & {307}$^\star$ & {639}$^\star$ & {1\,461}$^\star$ & 2\,923 & 7\,063 & 14\,127 & 35\,133 & 70\,267 & 180\,607 \\
& Time & 9 s & 259 s & 3 h & - & - & - & - & - & - \\
\cmidrule(rl){2-11} 
& Ours & 307$^\star$ & 639$^\star$ & 1\,461$^\star$ & \textbf{3\,125} & \textbf{7\,521} & \textbf{15\,963} & \textbf{37\,496} & \textbf{88\,257} & \textbf{199\,719} \\
& Time (Hit / Total) & 139 s / --  & 303 s / --  & 915 s / -- & 772 s / 2445 s & 2 h / 2 h & 2 h / 3 h & 4 h / 5 h & 5 h / 5 h & 8 h / 13 h \\
& $N$ (Hit / Total)  & 20 K / -- & 49 K  / -- & 110 K  / -- & 83 K / 195 K & 258 K / 258 K & 202 K / 265 K & 208 K / 240 K & 124 K / 132 K & 121 K / 134 K\\
\bottomrule

\end{tabular}}
\caption{Comparison of the achieved values for the considered rigidity invariants.
The numbers marked by star ($\star$) are the best possible,
otherwise, the numbers indicate the best found value.
\textbf{The bold values are new bounds, i.e., they were not previously discovered.}
\textcolor{gray}{The gray timings are lower bound estimates (see the text for details).}
We provide two timings for our method: the time at which the algorithm first identified the best-performing graph, and the total duration of the experiment, including the search time spent after the best graph was found. We indicate also the number of non-isomorphic graphs ($N$) for which the rigidity invariants were evaluated.
}
\label{tab:results}
\end{table}
\end{landscape}
\restoregeometry

In contrast to the planar case, for $15 \le n \le 18$ we find graphs with larger $\numSphereInv$ than~\cite{grasegger2025explorations}, \emph{establishing new best-known lower bounds for the maximum value of this invariant}. The $14$-vertex graph we found is isomorphic to the one reported in~\cite{grasegger2025explorations}. It is known that the maximum can be attained by several graphs. We found two different graphs on $15$ with the indicated number of spherical realizations, whereas for $n\ge 16$ we found a single new graph attaining our best value. Previously, multiple graphs attaining to the maximum have been reported only for $n\leq 13$, i.e., coming from exhaustive search.
The previously reported best graphs are Hamiltonian and have chromatic number three, minimum degree three and maximum degree four.
While for those the search might have been biased towards having these properties,
the graphs we have found possess them as well without imposing them.
Previously, the graphs contained at least two $3$-cycles, while
one of our $15$-vertex graphs and those with $16$, $17$ and $18$ vertices have even stronger property:
every vertex is in a $3$-cycle.

Regarding $\numNACinv$, \cite{clinch2024nac} does not explicitly provide graphs
with many NAC-colorings on $13$ to $17$ vertices.
Hence, we constructed them using the same method as the one with $18$ they provide.
Our method gives better graphs for all $13\leq n \leq 18$.
As expected, none of the graphs contains a $3$-cycle since its edges have to be of the same color.
Interestingly, our graphs on $13, 15$ and $16$ vertices can be constructed
from the complete bipartite graph on $3+3$ vertices using only $0$-extensions,
and we have observed such graphs also on $14$ and $17$ vertices
that have $\numNACinv$ close to the found maximum. 
The fact that these graphs can be constructed in similar manner could be potentially used
to obtain an infinite family of graphs with many NAC-colorings.
The graphs with $14$, $17$ and $18$ have non-trivial automorphism groups
and minimum degree three (i.e., the last step is a $1$-extension).
Instantiating \cite[Lemma~5.6]{clinch2024nac} with the $18$-vertex graph,
we get a new asymptotic lower bound $2.144^n$ on the maximum number of NAC-colorings.

We provide the certificate graphs for the improved bounds in \Cref{sec:certificate_graphs} and our repository~\cite{project-repo}.

\section{Related Works}

A fundamental work by~\cite{wagner2021constructions} introduced a general \textbf{reinforcement-learning framework for combinatorial constructions}, modeling them as sequential decision processes and showing that neural policies can effectively guide search toward high-quality mathematical objects. Building on this line of research, \cite{angileri2025analyzing} analyzed Wagner’s framework on Brouwer’s conjecture. Like us, they experimented with a Deep CEM-style scheme, but they target conjecture-driven graph search rather than rigidity optimization and do not require permutation-equivariant, extension-level action representation like ours. Furthermore, works such as \cite{patterboost} and \cite{az_extremal} propose effective methods for searching large extremal graphs (e.g., for Erd\H{o}s-type conjectures), using sophisticated architectures, including transformers, AlphaZero integrated with Pairformer modules, and metaheuristics like incremental Tabu search. However, these approaches are not well suited to our setting: they typically require generating and evaluating very large numbers of candidate graphs, which would be prohibitively here because computing rigidity invariants is the main computational bottleneck. More recently, \cite{georgiev2025mathatscale} demonstrated that combining evolutionary methods with large language models can enable mathematical discovery across a wide range of domains, albeit at the cost of relying on potentially expensive language-model infrastructure. In contrast, our approach is specialized to rigidity-theory invariants and provides an efficient and accurate search procedure without dependence on external models or sources.

To our best knowledge, no comparable machine learning-based approach has yet been applied to \textbf{optimizing rigidity invariants}. In a related vein, while we focus on complex realizations, \cite{bartzos2021maxreal} studied the number of real realizations using a heuristic-based search. For \emph{complex realization counts}, prior work has relied either on exhaustive search for small $n$~\cite{capco2024lnumber} or on large-scale computations guided by expert insight~\cite{GraseggerKoutschanTsigaridas,grasegger2025explorations}, similarly for $\numNACinv$ \cite{clinch2024nac}.

\section{Conclusion}

We introduced a learning approach to the extremal problem of finding minimally rigid graphs with large rigidity invariants. 
Empirically, we (i) recover the known optimal values for planar realization counts and (ii) establish improved best-known lower bounds for spherical realization counts (for $15\le n\le 18$) and for the number of NAC-colorings (for $13\le n\le 18$), while evaluating only a tiny fraction of all candidate graphs. Future work will focus on the structural features of the discovered extremal graphs and leverage them to derive new theoretical insights in rigidity theory and the extension of our framework to other invariants and larger $n$.

\section*{Acknowledgments}

This work was supported by the Student Summer Research Program 2025 of FIT CTU in Prague.
J.\,L.\ was supported by the Czech Science Foundation (GAČR), project No.\ 22-04381L.

\bibliographystyle{alphaurl}
\bibliography{refs_extended}

\clearpage

\newpage
\newpage

\appendix

\setcounter{figure}{0}
\setcounter{table}{0}
\setcounter{equation}{0}
\renewcommand{\thefigure}{A-\arabic{figure}}
\renewcommand{\thetable}{A-\arabic{table}}
\renewcommand{\theequation}{A-\arabic{equation}}

\section{Table of Notations}

We provide a comprehensive description of the notation in \Cref{tab:notation}.

\section{Certificate Graphs}
\label{sec:certificate_graphs}
The graphs certifying the new bounds reported in \Cref{tab:results}
for $\numSphereInv$ are in \Cref{tab:graphs_sphere} and \Cref{fig:graphs_sphere}.
The graphs for $\numNACinv$ are in \Cref{tab:graphs_NAC,tab:graphs_NAC_clinch} and \Cref{fig:graphs_NAC}.

We encode a graph by concatenating the rows of the upper triangle of its adjacency matrix
(excluding the diagonal) and interpreting the obtained binary string as an integer.
For decoding, for instance the method \texttt{Graph.from\_int} in \textsc{PyRigi}~\cite{pyrigi}
can be used.

\section{Reproducibility}
\label{sec:rep}

\paragraph{Implementation and Hardware Setup.} We implemented our method using PyTorch \cite{PyTorch} and PyTorch Geometric~\cite{fey2019pyg}. We used Hydra~\cite{Yadan2019Hydra} to manage hyperparameters and configurations, and MLflow~\cite{Zaharia_Accelerating_the_Machine_2018} to monitor experiments, store results, and track model performance. Our experiments were conducted on a virtualized machine with four processors Intel Xeon Gold 6254 providing (in total) 64 physical cores (hyper-threading disabled), a four-node NUMA architecture, and a base frequency of 3.10\,GHz with AVX-512 support, and with 32\,GB of RAM. All runtimes reported in our results table (see \Cref{tab:results} in the main paper) were measured on this same setup for both our method and the baselines. Although our approach is a deep learning method and would typically benefit from GPU acceleration, the computational bottleneck in our experiments is evaluating the invariants $\mathcal{R}(G)$. This evaluation is not significantly faster on GPUs, so we ran all experiments on CPUs.

\paragraph{Hyperparameter Selection.} The hyperparameters used in our experiments are as follows. We set the \textbf{evolutionary search parameters} \emph{to match our available hardware}. Unless stated otherwise, we run Deep CEM for up to $T=250$ generations (and $T=500$ for $\numNACinv$) with population size $m=1000$. For $(\mathcal{R},n)=(\numPlaneInv,18)$ we use $T<250$ due to the higher evaluation cost. For larger instances (e.g., $\numPlaneInv$ at $n=18$), we also apply an early-stopping strategy based on the rate of discovering new non-isomorphic graphs as a proxy for search productivity: if the number of previously unseen non-isomorphic graphs produced at generation $t$ falls below a threshold (500 for $\numPlaneInv$ and $\numSphereInv$, and 250 for $\numNACinv$), we stop the run early. This criterion is motivated by the empirical observation that once the discovery rate drops to such levels, the best invariant value typically plateaus and remains stable (see Figure~\ref{fig:early_stopping}). In each generation, we selected an elite fraction of $\rho_{\text{elite}}=0.064$ and retained a survivor fraction of $\rho_{\text{surv}}=0.016$ for the next generation. To reduce the cost of reward evaluation, we computed the chosen rigidity invariant for a fraction $\rho_{\text{main}}=0.256$ of graphs with the highest $m$-B\'ezout$(G)$ scores (setting $\rho_{\text{main}}=1$ for $\numNACinv$).

\begin{figure}
    \centering
    \includegraphics[width=0.8\linewidth]{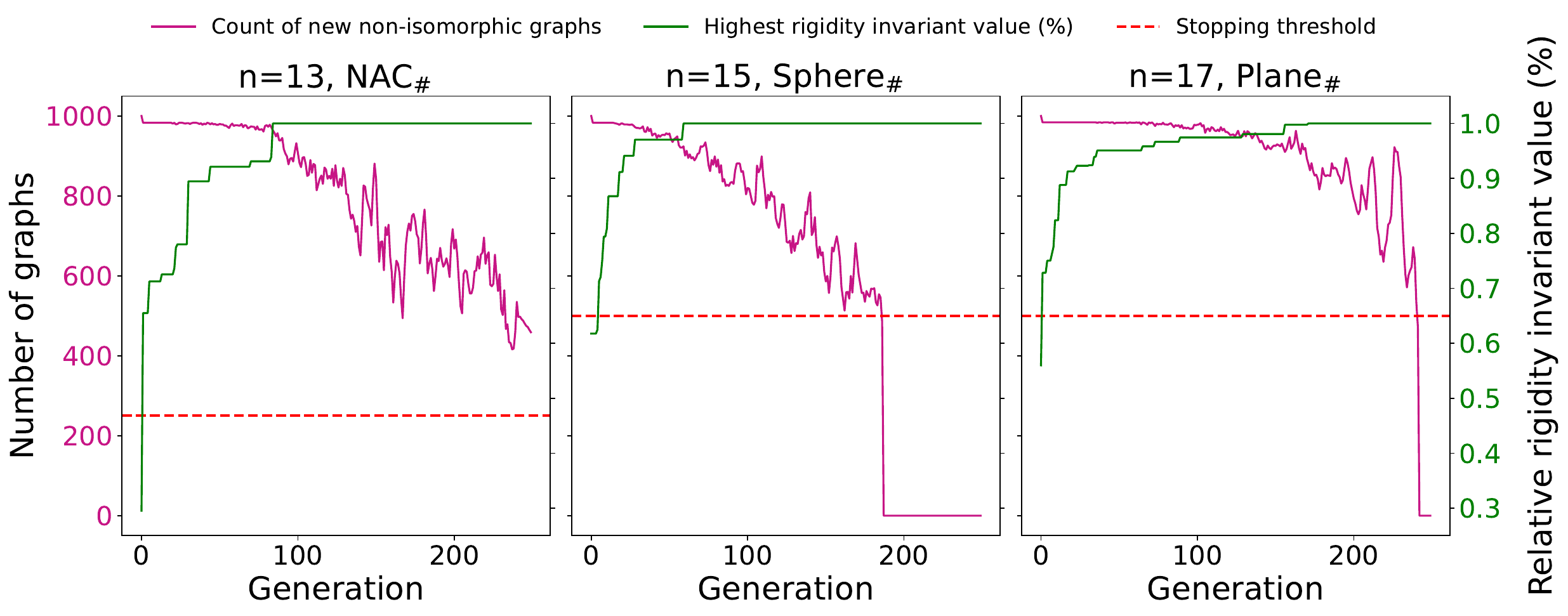}
    \caption{{Search dynamics and early stopping.} For three representative runs ($n{=}13$ for $\numNACinv$, $n{=}15$ for $\numSphereInv$, $n{=}17$ for $\numPlaneInv$), we plot per generation $t$ the \textcolor{magenta}{\textbf{number of newly discovered non-isomorphic graphs}} (left axis) alongside the \textcolor{darkgreen}{\textbf{best-so-far rigidity invariant value}} (right axis; shown on a relative scale). The red dashed line marks the \textcolor{red}{\textbf{discovery-rate threshold}} (250 for $\numNACinv$, 500 for $\numPlaneInv$ and $\numSphereInv$). For large graphs, runs are terminated once the discovery rate drops below this threshold, as subsequent reward improvements were empirically rare.}

    \label{fig:early_stopping}
\end{figure}

\begin{figure}
    \centering
    \includegraphics[width=0.8\linewidth]{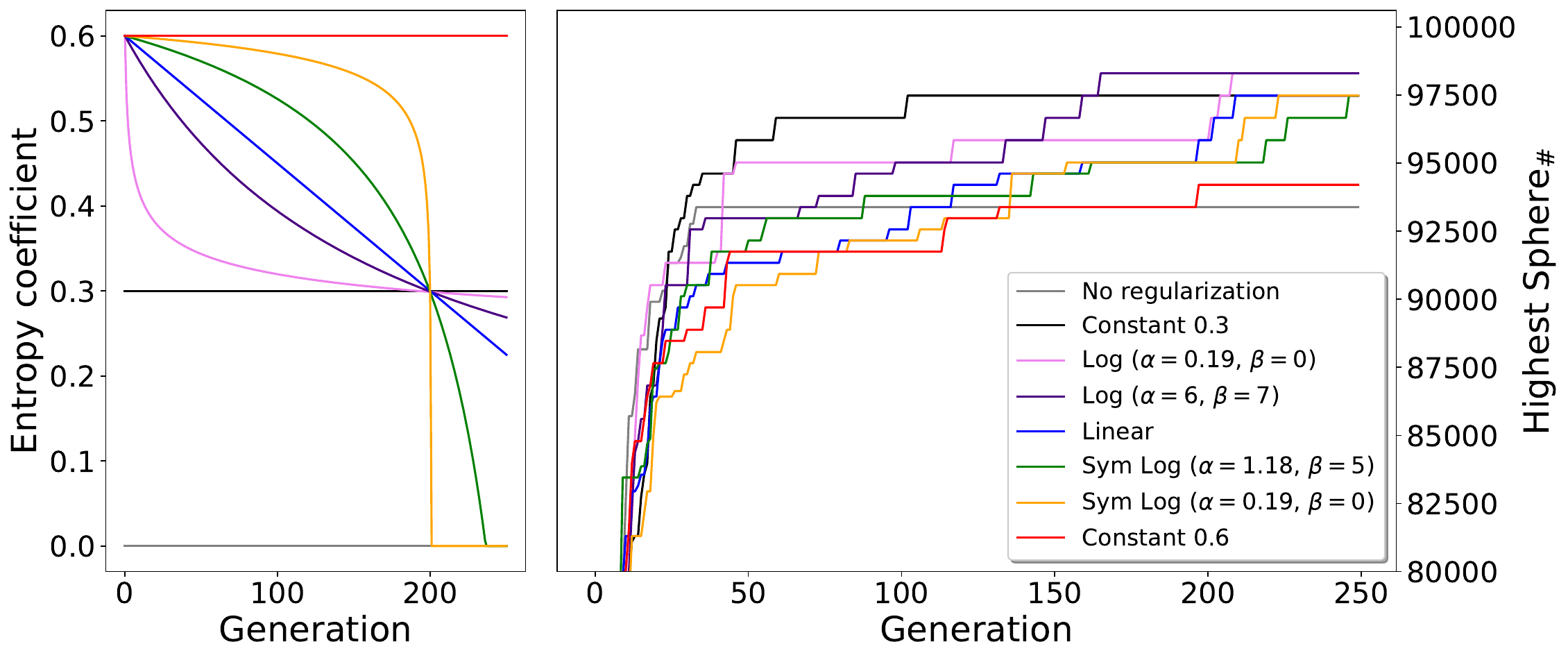}
    \caption{{Effect of entropy-regularization schedules when maximizing $\mathbf{\numSphereInv}$ (${n = 14}$).} \textbf{Left:} the entropy coefficient $\eta_t$ over generations for several decay schedules (including no regularization and constant $\eta$). \textbf{Right:} the \textbf{best $\numSphereInv$ achieved up to generation $t$} under each schedule, averaged over five independent runs.}

    \label{fig:explorations}
\end{figure}

\begin{table*}
\centering
{\footnotesize
\begin{tabular}{ll ccccc || ccccc}
\toprule
& & \multicolumn{5}{c}{To $\numPlaneInv$} & \multicolumn{5}{c}{To $\numSphereInv$} \\
\cmidrule(lr){3-7} \cmidrule(lr){8-12}
& $n$ & 10 & 11 & 12 & 13 & 14 & 10 & 11 & 12 & 13 & 14 \\
\midrule
\multirow{5}{*}{\rotatebox[origin=c]{90}{From} \rotatebox[origin=c]{90}{$\numPlaneInv$}} 
& 10 & \textbf{880} & 2\,276 & 5\,652 & 15\,234 & 38\,912 & \textbf{1\,536} & \textbf{4\,352} & 10\,240 & 28\,672 & 79\,872 \\
& 11 & \textbf{880} & \textbf{2\,288} & \textbf{6\,180} & \textbf{15\,536} & 39\,258 & \textbf{1\,536} & \textbf{4\,352} & 11\,264 & 31\,744 & 75\,776 \\
& 12 & \textbf{880} & 2\,276 & \textbf{6\,180} & 15\,268 & 40\,686 & \textbf{1\,536} & \textbf{4\,352} & 11\,776 & 27\,648 & 77\,824 \\
& 13 & \textbf{880} & 2\,276 & \textbf{6\,180} & \textbf{15\,536} & 40\,800 & \textbf{1\,536} & 4\,096 & 10\,752 & 29\,696 & 73\,728 \\
& 14 & \textbf{880} & 2\,276 & 5\,834 & \textbf{15\,536} & \textbf{42\,780} & \textbf{1\,536} & 3\,840 & 10\,752 & 28\,160 & 79\,360 \\
\midrule
\multirow{5}{*}{\rotatebox[origin=c]{90}{From} \rotatebox[origin=c]{90}{$\numSphereInv$}} 
& 10 & \textbf{880} & 2\,276 & 5640 & 15\,492 & 39\,336 & \textbf{1\,536} & \textbf{4\,352} & 11\,264 & 30\,720 & 79\,872 \\
& 11 & \textbf{880} & 2\,168 & 5748 & 14\,532 & 37\,586 & \textbf{1\,536} & \textbf{4\,352} & 11\,264 & 30\,720 & 86\,016 \\
& 12 & \textbf{880} & \textbf{2\,288} & 5\,952 & 14\,736 & 37\,744 & \textbf{1\,536} & \textbf{4\,352} & \textbf{12\,288} & 31\,744 & 88\,064 \\
& 13 & \textbf{880} & 2\,168 & 5\,704 & 14\,320 & 37\,084 & \textbf{1\,536} & \textbf{4\,352} & \textbf{12\,288} & \textbf{34\,816} & 92\,160 \\
& 14 & \textbf{880} & \textbf{2\,288} & 5\,952 & 15\,032 & 39\,000 & \textbf{1\,536} & \textbf{4\,352} & \textbf{12\,288} & \textbf{34\,816} & \textbf{98\,304} \\
\bottomrule
\end{tabular}
}
\caption{Transfer learning performance matrix. Rows correspond to models optimized to maximize either $\numPlaneInv$ or $\numSphereInv$ for a specific number of vertices $n$. Columns report the maximum reward ($\numPlaneInv$ or $\numSphereInv$) achieved when evaluating $10\,000$ non-isomorphic graphs of the size indicated by the column header. Models were trained until the empirical probability of generating a maximum-reward graph was at least $1$ in $1\,000$.}
\label{tab:transfer_learning_combined}
\end{table*}

Regarding the \textbf{entropy-regularization hyperparameters}, we empirically set them by inspecting the algorithm’s behavior on some small graphs, where the optimal values of the invariants are known. 
To illustrate the effect of entropy regularization on exploration, Figure~\ref{fig:explorations} compares several entropy schedules and their impact on the final performance when maximizing $\numSphereInv$ for $n=14$, averaged over five independent runs.
To make the differences between schedules more apparent, we evaluated only $6\%$ of the generated graphs (based on the surrogate values) using the main reward function. The results show that strong entropy regularization leads to underfitting, preventing the model from constructing the best graph, whereas removing entropy regularization altogether causes overfitting and premature convergence to a suboptimal policy. Only runs with a logarithmic decay of the entropy coefficient found the best graph in all observed runs. In contrast, runs without entropy regularization were faster, but did not consistently identified the best graph.

Thus, in \Cref{eq:ent_coef_decay}, we fixed $\alpha=6$ and $\beta=7$, and tuned $\eta_0$ as a function of the graph size (i.e., the number of vertices $n$) so that after learning for $T$ generations the model becomes sufficiently confident to repeatedly construct the best graph it has discovered. Empirically, to meet this constraint, with $\alpha=6$ and $\beta=7$, we observe that $\eta_0$ decreases approximately logarithmically with the size of the action space; we therefore calibrated $\eta_0$ on smaller instances and interpolated the resulting values for larger $n$.

For our \textbf{model architecture}, we set the number of layers of both the GIN and the extension-scoring head to $L_{\mathcal{G}}=L_{\mathcal{E}}=3$, and the dimension of the learnable step embedding to $d=2$. Each GIN layer uses an MLP with hidden layers of size $128$ and an output layer of size $32$. Likewise, the extension-scoring head is implemented with hidden layers of size $128$, but a single-neuron output layer that computes a logit for each considered extension in $\mathcal{E}$.  Since each~$e_{(u,\{v,w\})}$ is processed by the extension-representation and scoring module (an MLP with sufficient capacity) we instantiate the permutation-invariant set functions $\phi(\cdot)$ and $\psi(\cdot)$ directly as sum aggregators over their inputs.

Finally, we trained the model for four epochs on each batch of elite graphs (see Algorithm~1, line~14 in the main paper), using a learning rate of $lr=5\times 10^{-4}$. We also found that the choice of random seed was not critical: runs with different seeds and identical hyperparameter settings consistently converged to the same high-reward constructions, provided that exploration was sufficiently strong. For complete implementation details, obtained graphs, and reproducibility instructions, we refer to the project repository~\cite{project-repo}, which also reports the selected values of $\eta_0$ for different values of~$n$.

\section{Surrogate Efficiency and Variance Analysis}
\label{sec:surrogate}

The proposed approach was evaluated across multiple independent runs to verify its ability to consistently discover graphs with the maximum realization counts, and to assess the performance impact of utilizing $\mbezout(G)$ as a surrogate reward. As demonstrated in \Cref{tab:variance_results}, integrating the $\mbezout(G)$ yields an approximate $1.5\times$ computational speedup for the planar invariant ($\numPlaneInv$) and a $3\times$ speedup for the spherical invariant ($\numSphereInv$). Crucially, these efficiency gains do not compromise the search quality, as the algorithm successfully recovered the theoretically optimal invariant values across $10$ independent runs ($5$ runs for
$n = 15$).

\begin{table*}[ht]
\centering

% Define a custom command for the stacked mean and standard deviation.
\newcommand{\meanstd}[2]{\makecell{#1 \\[-1.2ex] \textcolor{gray}{\scriptsize $\pm #2$}}}

{\footnotesize
\begin{tabular}{clcccccc}
\toprule
& \# vertices $n$ &  10 & 11 & 12 & 13 & 14 & 15 \\
\midrule

% The [-1ex] fixes the vertical centering caused by the \addlinespace below
\multirow{5}{*}[-1ex]{\rotatebox[origin=c]{90}{$\numPlaneInv$}} &
Time (with m-Bézout)  & \meanstd{$63$}{14} & \meanstd{$160$}{103} & \meanstd{$262$}{162} & \meanstd{$637$}{252} & \meanstd{$3\,424$}{2\,429} & \meanstd{$10\,248$}{5\,627} \\
% The \hphantom acts as an invisible strut to force the column to the correct width
\hphantom{\rotatebox[origin=c]{90}{$\numPlaneInv$}} & 
$N$ (with m-Bézout)  & \meanstd{$2\,453$}{666} & \meanstd{$5\,043$}{3\,279} & \meanstd{$6\,093$}{3\,676} & \meanstd{$10\,573$}{4\,075} & \meanstd{$31\,923$}{21\,884} & \meanstd{$40\,243$}{20\,142} \\
\addlinespace
& Time (no m-Bézout)& \meanstd{$61$}{29} & \meanstd{$303$}{147} & \meanstd{$264$}{129} & \meanstd{$689$}{261} & \meanstd{$4\,815$}{2\,935} & \meanstd{$16\,635$}{6\,012} \\
& $N$ (no m-Bézout)& \meanstd{$6\,479$}{2\,415} & \meanstd{$23\,793$}{9\,367} & \meanstd{$21\,358$}{9\,264} & \meanstd{$37\,405$}{13\,369} & \meanstd{$113\,944$}{64\,097} & \meanstd{$123\,462$}{38\,328} \\

\midrule

\multirow{5}{*}[-1ex]{\rotatebox[origin=c]{90}{$\numSphereInv$}} &
Time (with m-Bézout) & \meanstd{$37$}{16} & \meanstd{$140$}{44} & \meanstd{$1\,398$}{636} & \meanstd{$11\,137$}{5\,572} & -- & -- \\
\hphantom{\rotatebox[origin=c]{90}{$\numSphereInv$}} & 
$N$ (with m-Bézout)  & \meanstd{$1\,382$}{736} & \meanstd{$3\,430$}{1\,127} & \meanstd{$14\,003$}{6\,471} & \meanstd{$28\,339$}{14\,418} & -- & -- \\
\addlinespace
& Time (no m-Bézout) & \meanstd{$43$}{16} & \meanstd{$239$}{93} & \meanstd{$3\,116$}{2\,002} & \meanstd{$36\,222$}{13\,056} & -- & -- \\
& $N$ (no m-Bézout)& \meanstd{$4\,308$}{1\,826} & \meanstd{$11\,609$}{4\,303} & \meanstd{$43\,879$}{27\,546} & \meanstd{$126\,314$}{46\,994} & -- & -- \\

\bottomrule
\end{tabular}}
\caption{Variance statistics (mean $\pm$ standard deviation) for execution time (in seconds) and the total number $N$ of evaluated graphs to reach the known optimum value across $10$~independent runs ($5$ runs for $n=15$). The table shows also the timings without using the m-Bézout bound approximation (25\% of graphs evaluated with true reward).}
\label{tab:variance_results}
\end{table*}

\section{Transfer Learning}
\label{sec:transfer_learning}
We also investigated the transfer-learning properties of our policy. Concretely, we trained a model to maximize a rigidity invariant $\mathcal{R}$ (e.g., $\mathcal{R}=\numPlaneInv$) on graphs with $n$ vertices, and then evaluated the same trained policy (without further fine-tuning) on a target invariant $\mathcal{R}'$ (e.g., $\mathcal{R}'=\numSphereInv$) with  graph size $n'$. \Cref{tab:transfer_learning_combined} summarizes these results: each row corresponds to a policy trained for a specific pair $(\mathcal{R},n)$, while columns report the best value achieved when deploying that policy to generate and evaluate $10{,}000$ non-isomorphic graphs for the target setting $(\mathcal{R}',n')$. 

We observe that \emph{in-domain} transfer (i.e., $\mathcal{R}=\mathcal{R}'$, such as $\numPlaneInv \to \numPlaneInv$ across different $n,n'$) is substantially stronger than \emph{cross-domain} transfer ($\numPlaneInv \to \numSphereInv$ and $\numSphereInv\to \numPlaneInv $). This has been expected since~\cite{grasegger2025explorations} reports that the graphs with maximum $\numPlaneInv$ do not necessarily maximize $\numSphereInv$ and the other way around.
Within the in-domain blocks, policies trained on larger instances often recover optimal values when evaluated on smaller graphs ($n>n'$): this holds for all transfers with $\mathcal{R}=\mathcal{R}'=\numSphereInv$, and for several cases with $\mathcal{R}=\mathcal{R}'=\numPlaneInv$. Interestingly, transfer appears more robust for the spherical invariant, even though $\numSphereInv$ is generally more computationally demanding. A plausible explanation is the two-stage evaluation strategy: the surrogate score m-B\'ezout$(G)$ correlates better (i.e., provides a tighter proxy) with $\numSphereInv$ than with $\numPlaneInv$, which may lead the policy to learn construction  that generalize more reliably across sizes.  In particular, $\text{m-Bézout}(G)$ upper-bounds the realization-count objectives, i.e.,
\[
\numPlaneInv(G) \leq \numSphereInv(G) \leq  \text{m-Bézout}(G)
\]
(see~\cite{DewarGrasegger} for the first inequality). Another important observation is that when transferring from a smaller  $n \in \{n_1,n_2\}$ with $n_1 < n_2 < n'$ and evaluating both policies on the same target size $n'$, the policy trained on $n_2$ generally discovers graphs with higher absolute rewards (although not necessarily optimal) than the policy trained on~$n_1$.

\begin{table*}
\centering
{\footnotesize
\begin{tabular}{llccccccccc}
\toprule
& \# vertices $n$ &  10 & 11 & 12 & 13 & 14 & 15 & 16 & 17 \\
\midrule

\multirow{3}{*}{\rotatebox[origin=c]{90}{$\numPlaneInv$}} &
\cite{grasegger2025explorations} & {880}$^\star$ & {2\,288}$^\star$ & {6\,180}$^\star$ & {15\,536}$^\star$ & {42\,780}$^\star$ & 112\,752 & 312\,636 & 877\,960 \\

& Ablated variant & {880}$^\star$ & {2\,288}$^\star$ & {6\,180}$^\star$ & {15\,536}$^\star$ & {42\,780}$^\star$ & 112\,752 & 308\,344 & 803\,432 \\

\cmidrule(rl){2-10} 
% \# tested graphs & & & & & & 7 M & 4 M & $<1$ M & 300 K\\[5pt]
& Ours  & 880$^\star$ & 2\,288$^\star$ & 6\,180$^\star$ & 15\,536$^\star$   & 42\,780$^\star$ & 112\,752 & 312\,636 & 877\,960  \\

\bottomrule

\end{tabular}}
\caption{Ablation of permutation-equivariance on $\numPlaneInv$. Stars ($^\star$) mark instances where the optimum is known from exhaustive enumeration.}

\label{tab:ablation}
\end{table*}

\section{Ablation}
\label{sec:ablation}

We implement our approach within the Deep CEM framework, using a stochastic policy network parameterized by a GIN graph encoder. The architecture is permutation-equivariant, so its action probabilities transform consistently under vertex relabelings of isomorphic graphs (as do the rigidity invariants considered in our study). Although equivariance is a natural inductive bias in this setting, we validate its practical impact via an ablation. Concretely, we fine-tune a non-equivariant Deep CEM baseline that takes as input a flattened upper-triangular adjacency representation and predicts an action distribution with a standard MLP. The fine-tuned MLP output is mapped to logits for all candidate $0$- and $1$-extensions and normalized with a softmax, matching the action space used by our main model. The results for $\numPlaneInv$ are shown in \Cref{tab:ablation}. We note that the ablated variant matches the optimum on smaller instances, but its performance degrades as $n$ grows (see $n\geq 16$). This suggests that explicitly enforcing permutation-equivariance becomes increasingly important as the action space and graph diversity expand (see also \Cref{sec:superexp}), and helps the policy generalize across isomorphic relabelings rather than overfitting to arbitrary vertex orderings.

Finally, we note that prior work~\cite{GraseggerKoutschanTsigaridas,grasegger2025explorations,clinch2024nac} relied on large-scale computations guided by expert insight. This is widely regarded as essential for obtaining strong bounds within practical time limits: naively sampling $0$- and $1$-extensions (or even with surrogate scores such as m-Bézout \emph{but without exploiting graph/extension structure}) is ineffective due to the enormous action space and the high cost of evaluating the resulting graphs. See also the discussion in the next section.

\section{Extension Impact}
\label{sec:extension_impact}

The effect on the number of realizations when applying a $0$- or $1$-extension
was studied in~\cite[Section~7.2]{grasegger2025explorations}.
While $\numSphereInv$ doubles when a $0$-extension is used, the factor might be in general both smaller
or greater than two for $1$-extensions. Since we see that the bounds always increase
more than by the factor of two, it suffices to consider only $1$-extensions.
The new best graphs we have found for $\numSphereInv$
cannot be obtained by $1$-extensions from each other, namely,
no graph on $n$ vertices can be obtained from a graph on $n-1$ vertices for $15 \leq n \leq 17$.
We have also computed all possible $1$-extensions of the best graphs on $14$ to $16$ vertices,
but none of them achieves the values we found using our model, see \Cref{tab:graphs_sphere_extensions}.

The fact that the best graphs are not $1$-extension of each other
could possibly explain that within domain transfer learning from smaller to larger $n$ usually does not
achieve the best values.

While the effect of $0$-extensions is well-understood for realization counts,
for $\numNACinv$ it can vary; hence, $0$-extensions have to be considered as well in this case.
For the new graphs we found, we again have that the graph on $n$ vertices cannot be obtained
by $0$- nor $1$-extension from the graph on $n-1$ vertices for $13\leq n \leq 18$.
Applying all possible $0$- and $1$-extensions to the graph on $n$ vertices
does not yield a graph with higher $\numNACinv$ either, see \Cref{tab:graphs_NAC_extensions}.

\section{Number of Minimally Rigid Graphs}
\label{sec:superexp}
The number of non-isomorphic minimally rigid graphs grows super-exponentially with the number of vertices,
since all minimally rigid graphs are exactly those constructed using $0$-extensions and $1$-extensions from $K_2$
(see e.g.~\cite[Theorems~19.2 and 19.11]{handbook}) and we prove that actually it is the case even when considering
only graphs constructed by $0$-extensions.

\begin{proposition}
    \label{final_bound}
    If $c_n$ is the number of non-isomorphic minimally rigid graphs with $n$ vertices that can be constructed using $0$-extensions starting from the graph $K_2$, then
    \[
    c_n \geq \frac{(n-2)!}{n2^{n-2}}\,.
    \]
    Particularly, $c_n$ grows super-exponentially.
\end{proposition}

\begin{proof}
    Let $a_n$ be the number of all labeled minimally rigid graphs with vertices $1,\ldots, n$ constructed using $0$-extensions by iteratively adding vertices in the order given by the labels. Since we choose a pair of vertices from $\{1, \ldots, i\}$ for the $0$-extension adding vertex $i+1$,
    we have
    \begin{align*}
            a_n &= \prod_{i=2}^{n-1} \binom{i}{2}
            = \prod_{i=2}^{n-1} \frac{i(i-1)}{2} \\
            % &= \frac{(n-1)(n-2)}{2} \cdot \frac{(n-2)(n-3)}{2} \cdots \frac{2 \cdot 1}{2} \\
            &= \prod_{i=2}^{n-1}i \cdot \prod_{i=2}^{n-1}(i-1) \cdot \prod_{i=2}^{n-1}\frac{1}{2} \\
            &= (n-1)! \cdot (n-2)! \cdot \frac{1}{2^{n-2}}
    \end{align*}
    Note that choosing a different pairs of vertices, say $\{u,v\}$ and $\{u',v'\}$ with $u\notin \{u',v'\}$,
    when adding vertex $i+1$ yields different labeled graphs, since for the first choice the graph contains the edge $\{u,i+1\}$,
    but it is not the case for the second choice.
    
    If $b_n$ is the number of all labeled minimally rigid graphs with vertices $1,\ldots, n$ constructed using $0$-extensions
    (not necessarily following the order given by the labels), then
    \[
        a_n \leq b_n \leq n! \cdot c_n \,,
    \]
    where the second inequality follows from the fact that there are at most $n!$
    different labeled graphs isomorphic to each other.
    Using the expression for $a_n$ and dividing by $n!$ yields the desired inequality.

    For any $K>1$, we have $c_n \in \omega(K^n)$ since
    \[
        \lim_{n\rightarrow\infty} \frac{(n-2)!}{n2^{n-2}} \cdot \frac{1}{K^n} = \infty\,, 
    \]
    by the ratio test. 
    This concludes the proof of the statement.
\end{proof}

% \bibliographystyleS{named}
% \bibliographyS{refs}

\newpage

\begin{table*}[ht!]
    \centering
    \begin{tabular}{crr}
        $n$ & Integer representation of $G$ & $\numSphereFun{G}$ \\
        \midrule
         15 & 2000828459594098240497450525056 & 278\,528 \\
            & 22185205662832118156851245393968 & 278\,528 \\
         16 & 676317030175026185879559871219632902 & 819\,200 \\
         17 & 1708810961581179146514778090735835808768 & 2\,228\,224 \\
         18 & 5717703424785600896298030199603140199580763136 & 6\,127\,616
    \end{tabular}
    \caption[The integer representations of graphs maximizing $\numSphereInv$]{The integer representations of graphs on $15$ to $18$ vertices certifying the obtained bounds for $\numSphereInv$.}
    \label{tab:graphs_sphere}
\end{table*}

\begin{table*}[ht!]
    \centering
    \begin{tabular}{crrr}
        $n$ & Integer representation of $G'$ & $\numSphereFun{G'}$  & best $\numSphereInv$ \\
        \midrule
         15 & 755920348961494135657743057152 & 245\,760 & 278\,528 \\
         16 & 65557947275202461973507167815775232 & 688\,128 & 819\,200 \\
         17 & 44322443031028970019396990995970331435057 & 2\,113\,536 & 2\,228\,224 \\
    \end{tabular}
    \caption{The integer representations of graphs $G'$ on $15$ to $17$ vertices
    with the maximum $\numSphereInv$ among all graphs obtained by a $1$-extension
    from the best found graph $G$ (see \protect\Cref{tab:graphs_sphere}) with one vertex less.
    }
    \label{tab:graphs_sphere_extensions}
\end{table*}
\begin{figure*}
    \centering
        \begin{tabular}{ccccc}
% graph with 15 vertices:
\begin{tikzpicture}[scale=0.006]
    \node[smallvertex] (0) at (0, 210) {};
    \node[smallvertex] (1) at (60, -180) {};
    \node[smallvertex] (2) at (210, 30) {};
    \node[smallvertex] (3) at (-210, 30) {};
    \node[smallvertex] (4) at (-60, -180) {};
    \node[smallvertex] (5) at (-120, 150) {};
    \node[smallvertex] (6) at (-30, 120) {};
    \node[smallvertex] (7) at (120, 60) {};
    \node[smallvertex] (8) at (-180, -90) {};
    \node[smallvertex] (9) at (120, 150) {};
    \node[smallvertex] (10) at (180, -90) {};
    \node[smallvertex] (11) at (30, 120) {};
    \node[smallvertex] (12) at (-120, 60) {};
    \node[smallvertex] (13) at (90, -90) {};
    \node[smallvertex] (14) at (-90, -90) {};
    \draw[thinedge] (0) to (5) (0) to (6) (0) to (9) (0) to (11) (1) to (4) (1) to (10) (1) to (11) (1) to (13) (2) to (7) (2) to (9) (2) to (10) (2) to (14) (3) to (5) (3) to (8) (3) to (12) (3) to (13) (4) to (6) (4) to (8) (4) to (14) (5) to (6) (5) to (12) (7) to (8) (7) to (9) (8) to (14) (9) to (11) (10) to (12) (10) to (13);
\end{tikzpicture}
&
% graph with 15 vertices:
\begin{tikzpicture}[scale=0.006]
    \node[smallvertex] (0) at (120, -30) {};
    \node[smallvertex] (1) at (210, -30) {};
    \node[smallvertex] (2) at (30, 120) {};
    \node[smallvertex] (3) at (-210, -30) {};
    \node[smallvertex] (4) at (-150, -120) {};
    \node[smallvertex] (5) at (-60, -180) {};
    \node[smallvertex] (6) at (150, -120) {};
    \node[smallvertex] (7) at (120, 180) {};
    \node[smallvertex] (8) at (60, -180) {};
    \node[smallvertex] (9) at (0, 210) {};
    \node[smallvertex] (10) at (180, 90) {};
    \node[smallvertex] (11) at (-120, 180) {};
    \node[smallvertex] (12) at (-180, 90) {};
    \node[smallvertex] (13) at (-30, 120) {};
    \node[smallvertex] (14) at (-120, -30) {};
    \draw[thinedge] (0) to (1) (0) to (5) (0) to (6) (1) to (2) (1) to (6) (1) to (10) (2) to (7) (2) to (9) (3) to (4) (3) to (12) (3) to (13) (3) to (14) (4) to (5) (4) to (7) (4) to (14) (5) to (8) (5) to (10) (6) to (8) (6) to (11) (7) to (9) (7) to (10) (8) to (12) (8) to (14) (9) to (11) (9) to (13) (11) to (12) (11) to (13);
\end{tikzpicture}
&
% graph with 16 vertices:
\begin{tikzpicture}[scale=0.006]
    \node[smallvertex] (0) at (-125, 175) {};
    \node[smallvertex] (1) at (-25, 225) {};
    \node[smallvertex] (2) at (75, 225) {};
    \node[smallvertex] (3) at (-125, -100) {};
    \node[smallvertex] (4) at (150, 50) {};
    \node[smallvertex] (5) at (-25, -150) {};
    \node[smallvertex] (6) at (-75, -25) {};
    \node[smallvertex] (7) at (-175, 100) {};
    \node[smallvertex] (8) at (200, 0) {};
    \node[smallvertex] (9) at (-175, 0) {};
    \node[smallvertex] (10) at (125, -25) {};
    \node[smallvertex] (11) at (150, 175) {};
    \node[smallvertex] (12) at (200, 100) {};
    \node[smallvertex] (13) at (75, -150) {};
    \node[smallvertex] (14) at (150, -100) {};
    \node[smallvertex] (15) at (25, -75) {};
    \draw[thinedge] (0) to (1) (0) to (7) (0) to (10) (1) to (2) (1) to (10) (1) to (15) (2) to (6) (2) to (7) (2) to (11) (3) to (4) (3) to (5) (3) to (6) (3) to (9) (4) to (8) (4) to (9) (4) to (12) (5) to (6) (5) to (13) (5) to (15) (7) to (9) (7) to (11) (8) to (10) (8) to (12) (8) to (14) (10) to (14) (11) to (12) (11) to (13) (13) to (14) (13) to (15);
\end{tikzpicture}
&
% graph with 17 vertices:
\begin{tikzpicture}[scale=0.006*375/460]
    \node[smallvertex] (0) at (220, -140) {};
    \node[smallvertex] (1) at (-20, 160) {};
    \node[smallvertex] (2) at (40, 240) {};
    \node[smallvertex] (3) at (-80, -140) {};
    \node[smallvertex] (4) at (-200, 100) {};
    \node[smallvertex] (5) at (-160, -140) {};
    \node[smallvertex] (6) at (180, 200) {};
    \node[smallvertex] (7) at (-60, -220) {};
    \node[smallvertex] (8) at (160, -40) {};
    \node[smallvertex] (9) at (260, 80) {};
    \node[smallvertex] (10) at (-140, 20) {};
    \node[smallvertex] (11) at (40, -60) {};
    \node[smallvertex] (12) at (-200, -40) {};
    \node[smallvertex] (13) at (160, 80) {};
    \node[smallvertex] (14) at (260, -40) {};
    \node[smallvertex] (15) at (-120, 200) {};
    \node[smallvertex] (16) at (140, -220) {};
    \draw[thinedge] (0) to (6) (0) to (8) (0) to (14) (0) to (16) (1) to (2) (1) to (5) (1) to (9) (1) to (15) (2) to (3) (2) to (6) (2) to (15) (3) to (5) (3) to (7) (4) to (10) (4) to (11) (4) to (12) (4) to (15) (5) to (7) (5) to (12) (6) to (9) (6) to (13) (7) to (11) (7) to (16) (8) to (10) (8) to (11) (8) to (14) (9) to (13) (9) to (14) (10) to (12) (10) to (13) (11) to (16);
\end{tikzpicture} 
&
\begin{tikzpicture}[scale=0.006*375/495]
    \node[smallvertex] (0) at (90, 263) {};
    \node[smallvertex] (1) at (9, 209) {};
    \node[smallvertex] (2) at (-188, 92) {};
    \node[smallvertex] (3) at (180, 115) {};
    \node[smallvertex] (4) at (262, 58) {};
    \node[smallvertex] (5) at (236, -106) {};
    \node[smallvertex] (6) at (-216, 178) {};
    \node[smallvertex] (7) at (-20, -161) {};
    \node[smallvertex] (8) at (-266, 35) {};
    \node[smallvertex] (9) at (156, -198) {};
    \node[smallvertex] (10) at (-126, -211) {};
    \node[smallvertex] (11) at (-187, -35) {};
    \node[smallvertex] (12) at (-78, 265) {};
    \node[smallvertex] (13) at (148, -96) {};
    \node[smallvertex] (14) at (47, -230) {};
    \node[smallvertex] (15) at (187, 214) {};
    \node[smallvertex] (16) at (-232, -139) {};
    \node[smallvertex] (17) at (-112, 181) {};
    \draw[thinedge] (0) to (1) (0) to (11) (0) to (12) (0) to (15) (1) to (7) (1) to (12) (1) to (13) (2) to (4) (2) to (6) (2) to (8) (3) to (4) (3) to (9) (3) to (15) (3) to (17) (4) to (5) (4) to (15) (5) to (9) (5) to (10) (5) to (13) (6) to (8) (6) to (12) (6) to (17) (7) to (10) (7) to (11) (7) to (14) (8) to (11) (8) to (16) (9) to (13) (9) to (14) (10) to (14) (10) to (16) (11) to (16) (12) to (17);
\end{tikzpicture}
        \end{tabular}
    \caption[Graphs maximizing $\numSphereInv$]{The graphs on $15$ to $18$ vertices certifying the obtained bounds for $\numSphereInv$.
    The two left graphs have both 15 vertices.}
    \label{fig:graphs_sphere}
\end{figure*}
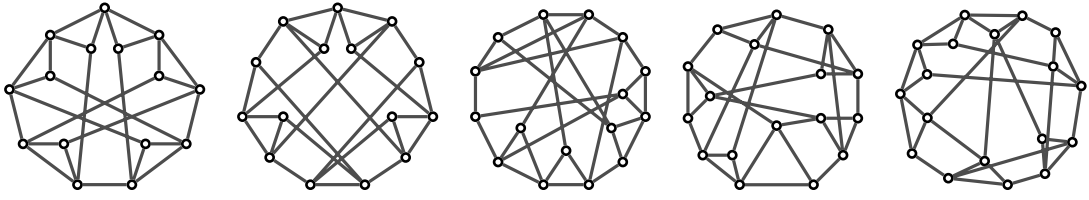

\clearpage

\begin{table*}
    \centering
    \begin{tabular}{crr}
        $n$ & Integer representation of $G$ & $\numNACfun{G}$ \\
        \midrule
         13 & 1817372602634323920930 & 3\,125 \\
         14 & 2178541080686613138604444182 & 7\,521 \\
         15 & 35514488197670496374812652340870 & 15\,963 \\
         16 & 88454699302609837679256749570852374 & 37\,496 \\
         17 & 43646696667421322394332935806613331125984 & 88\,257 \\
         18 & 44879647396852278983534873867663098247119872 & 199\,719 \\
    \end{tabular}
    \caption{The integer representations of graphs on $13$ to $18$ vertices certifying the obtained bounds for $\numNACinv$.}
    \label{tab:graphs_NAC}
\end{table*}

\begin{figure*}
    \centering
    \begin{tabular}{cccccc}
        % graph with 13 vertices:
        % Graph.from_int(1817372602634323920930)
        % Graph.from_vertices_and_edges([0, 1, 2, 3, 4, 5, 6, 7, 8, 9, 10, 11, 12], [(0, 8), (0, 9), (1, 2), (1, 4), (1, 9), (1, 11), (2, 5), (2, 6), (2, 7), (2, 10), (3, 8), (3, 11), (3, 12), (4, 5), (4, 6), (4, 7), (4, 10), (5, 8), (6, 9), (6, 12), (7, 11), (9, 10), (10, 12)])
        % placement:
        % {0: [-60, -120], 1: [-30, 60], 2: [-30, 0], 3: [60, 120], 4: [-90, 0], 5: [-60, -60], 6: [-90, 60], 7: [0, -60], 8: [60, -120], 9: [-150, 0], 10: [-150, 60], 11: [30, 0], 12: [-120, 120]}
        \begin{tikzpicture}[scale=0.007]
        	\node[smallvertex] (0) at (-60, -120) {};
        	\node[smallvertex] (1) at (-30, 60) {};
        	\node[smallvertex] (2) at (-30, 0) {};
        	\node[smallvertex] (3) at (60, 120) {};
        	\node[smallvertex] (4) at (-90, 0) {};
        	\node[smallvertex] (5) at (-60, -60) {};
        	\node[smallvertex] (6) at (-90, 60) {};
        	\node[smallvertex] (7) at (0, -60) {};
        	\node[smallvertex] (8) at (60, -120) {};
        	\node[smallvertex] (9) at (-150, 0) {};
        	\node[smallvertex] (10) at (-150, 60) {};
        	\node[smallvertex] (11) at (30, 0) {};
        	\node[smallvertex] (12) at (-120, 120) {};
        	\draw[thinedge] (0) to (8) (0) to (9) (1) to (2) (1) to (4) (1) to (9) (1) to (11) (2) to (5) (2) to (6) (2) to (7) (2) to (10) (3) to (8) (3) to (11) (3) to (12) (4) to (5) (4) to (6) (4) to (7) (4) to (10) (5) to (8) (6) to (9) (6) to (12) (7) to (11) (9) to (10) (10) to (12);
        \end{tikzpicture}
        &
        % graph with 14 vertices:
        % Graph.from_int(2178541080686613138604444182)
        % Graph.from_vertices_and_edges([0, 1, 2, 3, 4, 5, 6, 7, 8, 9, 10, 11, 12, 13], [(0, 1), (0, 2), (0, 3), (0, 8), (0, 10), (1, 4), (1, 5), (1, 13), (2, 5), (2, 7), (3, 4), (3, 5), (3, 13), (4, 6), (5, 8), (5, 10), (6, 7), (6, 9), (7, 11), (8, 9), (8, 12), (9, 10), (10, 12), (11, 12), (11, 13)])
        % placement:
        % {0: [-60, -180], 1: [-180, -90], 2: [0, 30], 3: [-180, -30], 4: [60, 30], 5: [60, -180], 6: [60, 120], 7: [0, 150], 8: [180, -30], 9: [120, 60], 10: [180, -90], 11: [-60, 120], 12: [-60, 30], 13: [-120, 60]}
        \begin{tikzpicture}[scale=0.007*330/350]
        	\node[smallvertex] (0) at (-60, -180) {};
        	\node[smallvertex] (1) at (-180, -90) {};
        	\node[smallvertex] (2) at (0, 60) {};
        	\node[smallvertex] (3) at (-180, -30) {};
        	\node[smallvertex] (4) at (60, 30) {};
        	\node[smallvertex] (5) at (60, -180) {};
        	\node[smallvertex] (6) at (60, 120) {};
        	\node[smallvertex] (7) at (0, 150) {};
        	\node[smallvertex] (8) at (180, -30) {};
        	\node[smallvertex] (9) at (120, 60) {};
        	\node[smallvertex] (10) at (180, -90) {};
        	\node[smallvertex] (11) at (-60, 120) {};
        	\node[smallvertex] (12) at (-60, 30) {};
        	\node[smallvertex] (13) at (-120, 60) {};
        	\draw[thinedge] (0) to (1) (0) to (2) (0) to (3) (0) to (8) (0) to (10) (1) to (4) (1) to (5) (1) to (13) (2) to (5) (2) to (7) (3) to (4) (3) to (5) (3) to (13) (4) to (6) (5) to (8) (5) to (10) (6) to (7) (6) to (9) (7) to (11) (8) to (9) (8) to (12) (9) to (10) (10) to (12) (11) to (12) (11) to (13);
        \end{tikzpicture}
        &
        % graph with 15 vertices:
        % Graph.from_int(35514488197670496374812652340870)
        % Graph.from_vertices_and_edges([0, 1, 2, 3, 4, 5, 6, 7, 8, 9, 10, 11, 12, 13, 14], [(0, 1), (0, 2), (0, 3), (0, 11), (1, 4), (1, 5), (2, 4), (2, 5), (2, 7), (2, 12), (3, 4), (3, 5), (3, 7), (3, 12), (4, 8), (4, 11), (5, 8), (6, 9), (6, 10), (7, 13), (7, 14), (8, 9), (9, 14), (10, 11), (10, 13), (12, 13), (12, 14)])
        % placement:
        % {0: [-60, 120], 1: [-180, 180], 2: [-120, 180], 3: [-60, 180], 4: [-120, 120], 5: [-180, 120], 6: [-30, -60], 7: [30, 240], 8: [-150, 60], 9: [-90, 0], 10: [30, 0], 11: [-90, 60], 12: [-30, 240], 13: [60, 60], 14: [0, 60]}
        \begin{tikzpicture}[scale=0.007]
        	\node[smallvertex] (0) at (-60, 120) {};
        	\node[smallvertex] (1) at (-180, 180) {};
        	\node[smallvertex] (2) at (-120, 180) {};
        	\node[smallvertex] (3) at (-60, 180) {};
        	\node[smallvertex] (4) at (-120, 120) {};
        	\node[smallvertex] (5) at (-180, 120) {};
        	\node[smallvertex] (6) at (-30, -60) {};
        	\node[smallvertex] (7) at (30, 240) {};
        	\node[smallvertex] (8) at (-150, 60) {};
        	\node[smallvertex] (9) at (-90, 0) {};
        	\node[smallvertex] (10) at (30, 0) {};
        	\node[smallvertex] (11) at (-90, 60) {};
        	\node[smallvertex] (12) at (-30, 240) {};
        	\node[smallvertex] (13) at (60, 60) {};
        	\node[smallvertex] (14) at (0, 60) {};
        	\draw[thinedge] (0) to (1) (0) to (2) (0) to (3) (0) to (11) (1) to (4) (1) to (5) (2) to (4) (2) to (5) (2) to (7) (2) to (12) (3) to (4) (3) to (5) (3) to (7) (3) to (12) (4) to (8) (4) to (11) (5) to (8) (6) to (9) (6) to (10) (7) to (13) (7) to (14) (8) to (9) (9) to (14) (10) to (11) (10) to (13) (12) to (13) (12) to (14);
        \end{tikzpicture}
        &
        % graph with 16 vertices:
        % Graph.from_int(88454699302609837679256749570852374)
        % Graph.from_vertices_and_edges([0, 1, 2, 3, 4, 5, 6, 7, 8, 9, 10, 11, 12, 13, 14, 15], [(0, 4), (0, 8), (0, 13), (1, 2), (1, 5), (1, 8), (1, 9), (1, 10), (1, 12), (2, 3), (2, 7), (2, 15), (3, 5), (3, 8), (3, 9), (3, 10), (3, 12), (4, 6), (5, 11), (6, 7), (6, 11), (7, 10), (9, 11), (9, 14), (10, 15), (11, 12), (12, 14), (13, 14), (13, 15)])
        % placement:
        % {0: [-120, -270], 1: [-120, -90], 2: [-30, -150], 3: [-60, -90], 4: [-60, -270], 5: [0, -30], 6: [30, -210], 7: [-90, -210], 8: [-150, -150], 9: [-60, -30], 10: [-90, -150], 11: [0, -90], 12: [-120, -30], 13: [-210, -150], 14: [-150, 30], 15: [-150, -210]}
        \begin{tikzpicture}[scale=0.007]
        	\node[smallvertex] (0) at (-120, -270) {};
        	\node[smallvertex] (1) at (-120, -90) {};
        	\node[smallvertex] (2) at (-30, -150) {};
        	\node[smallvertex] (3) at (-60, -90) {};
        	\node[smallvertex] (4) at (-60, -270) {};
        	\node[smallvertex] (5) at (-60, -30) {};
        	\node[smallvertex] (6) at (30, -210) {};
        	\node[smallvertex] (7) at (-90, -210) {};
        	\node[smallvertex] (8) at (-150, -150) {};
        	\node[smallvertex] (9) at (0, -30) {};
        	\node[smallvertex] (10) at (-90, -150) {};
        	\node[smallvertex] (11) at (0, -90) {};
        	\node[smallvertex] (12) at (-120, -30) {};
        	\node[smallvertex] (13) at (-210, -150) {};
        	\node[smallvertex] (14) at (-150, 30) {};
        	\node[smallvertex] (15) at (-150, -210) {};
        	\draw[thinedge] (0) to (4) (0) to (8) (0) to (13) (1) to (2) (1) to (5) (1) to (8) (1) to (9) (1) to (10) (1) to (12) (2) to (3) (2) to (7) (2) to (15) (3) to (5) (3) to (8) (3) to (9) (3) to (10) (3) to (12) (4) to (6) (5) to (11) (6) to (7) (6) to (11) (7) to (10) (9) to (11) (9) to (14) (10) to (15) (11) to (12) (12) to (14) (13) to (14) (13) to (15);
        \end{tikzpicture}
        &
        % graph with 17 vertices:
        % Graph.from_int(43646696667421322394332935806613331125984)
        % Graph.from_vertices_and_edges([0, 1, 2, 3, 4, 5, 6, 7, 8, 9, 10, 11, 12, 13, 14, 15, 16], [(0, 1), (0, 10), (0, 14), (1, 4), (1, 16), (2, 6), (2, 8), (2, 9), (2, 10), (2, 11), (2, 12), (3, 7), (3, 9), (3, 14), (4, 6), (4, 7), (5, 6), (5, 8), (5, 9), (5, 10), (5, 11), (5, 12), (7, 15), (8, 13), (8, 15), (11, 13), (11, 16), (12, 13), (12, 15), (12, 16), (13, 14)])
        % placement:
        % {0: [180, 30], 1: [180, -150], 2: [-30, -150], 3: [-180, 30], 4: [0, -240], 5: [30, -150], 6: [0, -180], 7: [-180, -150], 8: [-120, -90], 9: [-60, -60], 10: [60, -60], 11: [120, -90], 12: [0, -30], 13: [0, 30], 14: [0, 90], 15: [-90, 0], 16: [90, 0]}
        \begin{tikzpicture}[scale=0.007*330/350]
        	\node[smallvertex] (0) at (180, 30) {};
        	\node[smallvertex] (1) at (180, -150) {};
        	\node[smallvertex] (2) at (-30, -150) {};
        	\node[smallvertex] (3) at (-180, 30) {};
        	\node[smallvertex] (4) at (0, -240) {};
        	\node[smallvertex] (5) at (30, -150) {};
        	\node[smallvertex] (6) at (0, -180) {};
        	\node[smallvertex] (7) at (-180, -150) {};
        	\node[smallvertex] (8) at (-120, -90) {};
        	\node[smallvertex] (9) at (-60, -60) {};
        	\node[smallvertex] (10) at (60, -60) {};
        	\node[smallvertex] (11) at (120, -90) {};
        	\node[smallvertex] (12) at (0, -30) {};
        	\node[smallvertex] (13) at (0, 30) {};
        	\node[smallvertex] (14) at (0, 90) {};
        	\node[smallvertex] (15) at (-90, 0) {};
        	\node[smallvertex] (16) at (90, 0) {};
        	\draw[thinedge] (0) to (1) (0) to (10) (0) to (14) (1) to (4) (1) to (16) (2) to (6) (2) to (8) (2) to (9) (2) to (10) (2) to (11) (2) to (12) (3) to (7) (3) to (9) (3) to (14) (4) to (6) (4) to (7) (5) to (6) (5) to (8) (5) to (9) (5) to (10) (5) to (11) (5) to (12) (7) to (15) (8) to (13) (8) to (15) (11) to (13) (11) to (16) (12) to (13) (12) to (15) (12) to (16) (13) to (14);
        \end{tikzpicture}
        &
        % graph with 18 vertices:
        % Graph.from_int(44879647396852278983534873867663098247119872)
        % Graph.from_vertices_and_edges([0, 1, 2, 3, 4, 5, 6, 7, 8, 9, 10, 11, 12, 13, 14, 15, 16, 17], [(0, 8), (0, 16), (0, 17), (1, 4), (1, 5), (1, 9), (1, 11), (1, 12), (1, 14), (1, 15), (1, 16), (2, 6), (2, 10), (2, 11), (2, 17), (3, 4), (3, 5), (3, 9), (3, 11), (3, 12), (3, 14), (3, 15), (3, 16), (4, 13), (5, 6), (6, 13), (7, 8), (7, 10), (7, 14), (8, 9), (8, 13), (10, 15), (12, 17)])
        % placement:
        % {0: [105, -105], 1: [35, -35], 2: [-140, 35], 3: [35, 105], 4: [70, 140], 5: [0, 140], 6: [-70, 175], 7: [105, -175], 8: [210, 35], 9: [105, 70], 10: [-35, -175], 11: [-35, 70], 12: [-35, 0], 13: [140, 175], 14: [70, -70], 15: [0, -70], 16: [105, 0], 17: [-35, -105]}
        \begin{tikzpicture}[scale=0.007*300/350]
        	\node[smallvertex] (0) at (105, -105) {};
        	\node[smallvertex] (1) at (35, -35) {};
        	\node[smallvertex] (2) at (-140, 35) {};
        	\node[smallvertex] (3) at (35, 105) {};
        	\node[smallvertex] (4) at (70, 140) {};
        	\node[smallvertex] (5) at (0, 140) {};
        	\node[smallvertex] (6) at (-70, 175) {};
        	\node[smallvertex] (7) at (105, -175) {};
        	\node[smallvertex] (8) at (210, 35) {};
        	\node[smallvertex] (9) at (105, 70) {};
        	\node[smallvertex] (10) at (-35, -175) {};
        	\node[smallvertex] (11) at (-35, 70) {};
        	\node[smallvertex] (12) at (-35, 0) {};
        	\node[smallvertex] (13) at (140, 175) {};
        	\node[smallvertex] (14) at (70, -70) {};
        	\node[smallvertex] (15) at (0, -70) {};
        	\node[smallvertex] (16) at (105, 0) {};
        	\node[smallvertex] (17) at (-35, -105) {};
        	\draw[thinedge] (0) to (8) (0) to (16) (0) to (17) (1) to (4) (1) to (5) (1) to (9) (1) to (11) (1) to (12) (1) to (14) (1) to (15) (1) to (16) (2) to (6) (2) to (10) (2) to (11) (2) to (17) (3) to (4) (3) to (5) (3) to (9) (3) to (11) (3) to (12) (3) to (14) (3) to (15) (3) to (16) (4) to (13) (5) to (6) (6) to (13) (7) to (8) (7) to (10) (7) to (14) (8) to (9) (8) to (13) (10) to (15) (12) to (17);
        \end{tikzpicture}
    \end{tabular}
    \caption{The graphs on $13$ to $18$ vertices certifying the obtained bounds for $\numNACinv$.}
    \label{fig:graphs_NAC}
\end{figure*}

\begin{table*}
    \centering
    \begin{tabular}{crr}
        $n$ & $G$ & $\numNACfun{G}$ \\
        \midrule
        13 & 170363797095532441635376 & 2\,923 \\
        14 & 1395360292174978547951223617 & 7\,063 \\
        15 & 22859454182150718848230338095108 & 14\,127 \\
        16 & 749023707617915212187976649078898721 & 35\,133 \\
        17 & 49086874595737144883235931874747612135940 & 70\,267
    \end{tabular}
    \caption{The integer representations of graphs on $13$ to $17$ vertices constructed using the approach in \protect\cite{clinch2024nac}.}
    \label{tab:graphs_NAC_clinch}
\end{table*}

\begin{table*}
    \centering
    \begin{tabular}{crrr}
        $n$ & Integer representation of $G'$ & $\numNACfun{G'}$  & best $\numNACinv$ \\
        \midrule
         13 & 189565301677203464126464 & 2\,923  & 3\,125 \\
         14 & 14999728119619681459012112 & 6\,656 & 7\,521 \\
         15 & 35692752091932812077244995486002 & 15\,763 & 15\,963 \\
         16 & 1163731142807089982295744369657991216 & 37\,207 & 37\,496 \\
         17 & 5796200596001745654800091751920657580344 & 81\,042 & 88\,257 \\
         18 & 5720848934857304615415085872018485039436749312 &  184\,592 & 199\,719 \\
    \end{tabular}
    \caption{The integer representations of graphs $G'$ on $13$ to $18$ vertices
    with the maximum $\numNACinv$ among all graphs obtained by a $0$- or $1$-extension
    from the best found graph $G$ (see \protect\Cref{tab:graphs_NAC}) with one vertex less.
    We indicate only one graph $G'$ attaining the maximum, but for $n=13$, there are three of them.}
    \label{tab:graphs_NAC_extensions}
\end{table*}

\newpage
\clearpage
\pagestyle{empty}

\begin{table*}
\centering
{\scriptsize
\begin{tabular}{@{}l@{} p{0.86\linewidth}@{}}
\toprule
\textbf{Variable} & \textbf{Description} \\
\midrule

\multicolumn{2}{@{}l}{\textbf{Graphs, realizations, and invariants}}\\
$G=(V,E)$ & Simple undirected graph with vertex set $V$ and edge set $E$.\\
$n = |V|$ & Number of vertices of $G$.\\
$|E|$ & Number of edges of $G$.\\
$uv \in E$ & Undirected edge between vertices $u$ and $v$ (same as $vu\in E$).\\
$K_n$ & Clique (complete graph) on $n$ vertices.\\
$K_2$ & Base graph (clique on two vertices) used to start the extension construction.\\
$p\colon V\to \RR^2$ & Planar realization (placement) of the vertices in the plane.\\
$(x_u,y_u)$ & Coordinates of vertex $u$ in a realization variable assignment.\\
$\bar{u}\bar{v}\in E$ & Pinned edge used to factor out translations and rotations in \eqref{eq:distances}.\\
$\numPlaneFun{G}$ & Number of (complex) planar realizations of a minimally rigid graph $G$.\\
$\numSphereFun{G}$ & Number of (complex) spherical realizations of a minimally rigid graph $G$.\\
$\numNACfun{G}$ & Number of NAC-colorings of $G$ divided by $2$ (color swap quotient).\\
$\rigInv$ & Generic reward/invariant map $\rigInv:\nGraphs\to\nonnegReals$ (e.g., $\numPlaneInv,\numSphereInv,\numNACinv$).\\
$\nGraphs$ & Set of minimally rigid graphs on $n$ vertices.\\
$\text{m-B\'ezout}(G)$ & Efficiently computable upper bound used as a surrogate score during screening.\\

\addlinespace
\multicolumn{2}{@{}l}{\textbf{Extensions (actions) and construction process}}\\
$G_k=(V_k,E_k)\,\,$ & Intermediate graph after $k$ vertices have been constructed.\\
$0$-extension & Adds a new vertex $z$ and edges $uz,vz$ for distinct $u,v\in V_k$.\\
$1$-extension & Adds a new vertex $z$, removes an edge $vw\in E_k$, and adds edges $uz,vz,wz$ for distinct $u,v,w\in V_k$.\\
$\mathcal{S}$,\ $\mathcal{A}$ & State space (intermediate graphs) and action space (candidate extensions).\\
$\lambda_k$ & State at step $k$ (here: the intermediate graph $G_k$).\\
$a_k$ & Action at step $k$ (chosen extension applied to $G_k$).\\

\addlinespace
\multicolumn{2}{@{}l}{\textbf{Deep CEM / RL notation}}\\
$\pi_\theta$ & Stochastic policy network mapping states to distributions over actions.\\
$\Delta(\mathcal{A})$ & Probability simplex over the action space $\mathcal{A}$.\\
$\theta$ & Learnable parameters of the policy network.\\
$\theta_G,\ \theta_{\mathcal{E}}$ & Parameters of the GIN encoder and the extension-scoring head, respectively (Fig.~\ref{fig:policy}).\\
$m$ & Population size (number of constructions per generation).\\
$T$ & Number of generations in Deep CEM.\\
$P_t$ & Population of candidate graphs at generation $t$.\\
$S_t$ & Survivor set carried to generation $t+1$.\\
$\Omega_t$ & Elite set (top-performing candidates) at generation $t$.\\
$\rho_{\text{elite}}$ & Elite fraction selected in each generation.\\
$\rho_{\text{surv}}$ & Survivor fraction retained to the next generation.\\
$\rho_{\text{main}}$ & Fraction of top m-B\'ezout candidates evaluated with the main reward/invariant.\\
$\mathcal{D}$ & Dataset of state--action pairs extracted from elite rollouts.\\
$\mathcal{L}(\theta)$ & Training loss (negative log-likelihood plus entropy regularization), see \eqref{eq:entropy}.\\

\addlinespace
\multicolumn{2}{@{}l}{\textbf{GNN encoder and features}}\\
$L_{\mathcal{G}}$ & Number of GIN layers.\\
$L_{\mathcal{E}}$ & Number of layers in the extension-scoring head (MLP).\\
$h_v^{(l)}$ & Hidden representation of vertex $v$ at GIN layer $l$.\\
$\epsilon^{(l)}$ & Learnable scalar in the GIN aggregation at layer $l$.\\
$\mathrm{MLP}^{(l)}$ & Layer-specific MLP used inside the $l$-th GIN layer.\\
$\mathcal{N}(v)$ & Open neighborhood of vertex $v$.\\
$\deg(v)$ & Degree of vertex $v$.\\
$\LDP{v}$ & Local Degree Profile feature vector for vertex $v$.\\
$s_k\in\RR^{d}$ & Learnable step embedding (step-aware feature); $d$ is its dimension.\\
$\kappa_v$ & Clustering coefficient of $v$ (triangle density around $v$).\\
$\tau(v)$ & Number of triangles incident to vertex $v$.\\
$\mathcal{H}^{(0)}$ & Collection of initial vertex features $(h_v^{(0)})_{v\in V_k}$.\\

\addlinespace
\multicolumn{2}{@{}l}{\textbf{Extension representation and scoring}}\\
$\mathcal{E}$ & Extension-level representations.\\
$e_{(u,\{v,w\})}\in\mathcal{E}$ & Representation of a candidate extension indexed by a vertex tuple (validity encoded via indicators).\\
$\phi(\cdot),\ \psi(\cdot)$ & Permutation-invariant set functions used to build extension-level features from vertex embeddings.\\
$\gamma_{(u,v,w)}$ & Binary indicators (in $\{0,1\}^5$) for invalid / $0$-extension / $1$-extension subclasses.\\
$z_i$ & Logit score produced by the extension-scoring head for extension $i$.\\
$p_{(u,\{v,w\})}$ & Softmax probability assigned to the candidate extension indexed by $(u,\{v,w\})$.\\

\addlinespace
\multicolumn{2}{@{}l}{\textbf{Entropy regularization}}\\
$H(\pi_\theta(\cdot\mid s))$ & Entropy of the action distribution at state $s$.\\
$\eta$ & Entropy regularization coefficient in \eqref{eq:entropy}.\\
$\eta_t$ & Entropy coefficient at generation $t$ (decay schedule), see \eqref{eq:ent_coef_decay}.\\
$\eta_0,\ \alpha,\ \beta$ & Hyperparameters controlling the decay of $\eta_t$.\\

\bottomrule
\end{tabular}
}
\caption{Summary of the main notation used throughout the paper.}
\label{tab:notation}
\end{table*}

\thispagestyle{empty}

\end{document}